\documentclass{article}

\usepackage{graphicx}
\usepackage{booktabs}
\usepackage{caption}
\usepackage{subcaption}
\usepackage{showlabels}
\usepackage{enumitem}
\usepackage[dvipsnames]{xcolor}

\usepackage{floatrow}

\usepackage{hyperref}

\usepackage{graphbox, amsmath, array, amsthm, amssymb}

\usepackage{amsmath,amsfonts,bm}

\newtheorem{theorem}{Theorem}
\newtheorem{lemma}{Lemma}

\newtheorem{corollary}{Corollary}
\newtheorem{remark}{Remark}

\newcommand{\R}{\mathbb{R}}

\newcommand{\E}{\mathbb{E}}

\renewcommand{\a}{\alpha}

\newcommand{\e}{\varepsilon}
\newcommand{\calO}{\mathcal{O}}
\newcommand{\bfx}{\mathbf{x}}
\newcommand{\Flin}{{F^\text{lin}}}
\newcommand{\tr}{\mathrm{tr}}

\newcommand{\relu}{\mathrm{ReLU}}

\newcommand{\ttil}[1]{\Sigma_{#1}^{(a)}}
\newcommand{\calQ}{\mathcal{Q}^{(a)}}
\newcommand{\citep}{\cite}

\newcommand{\eq}[1]{\begin{equation} #1 \end{equation}}
\newcommand{\al}[1]{\begin{align} #1 \end{align}}
\newcommand{\als}[1]{\begin{align*} #1 \end{align*}}

\newcommand{\pa}[1]{\left( #1 \right)}
\newcommand{\qa}[1]{\left[ #1 \right]}
\newcommand{\absa}[1]{\left| #1 \right|}

\newcommand{\lem}[1]{\begin{lemma} #1 \end{lemma}}
\newcommand{\pf}[1]

\usepackage[utf8]{inputenc} 
\usepackage[T1]{fontenc}    
\usepackage{hyperref}       
\usepackage{url}            
\usepackage{booktabs}       
\usepackage{amsfonts}       
\usepackage{nicefrac}       
\usepackage{microtype}      
\usepackage{xcolor}         
\usepackage{hyperref}

\title{A Random Matrix Perspective on Mixtures of\\Nonlinearities in High Dimensions}

\author{%
\setlength\tabcolsep{5mm}
\hspace{-9mm}
\begin{tabular}{c c c}
    \textbf{Ben Adlam}\thanks{Equal contribution.}\; \thanks{Work done as a member of the Google AI Residency program (g.co/brainresidency).} & \textbf{Jake Levinson}\footnotemark[1]\; \footnotemark[2] & \textbf{Jeffrey Pennington}\footnotemark[1] \\
    Google Research & Google Research & Google Research \\
    \texttt{adlam@google.com} & \texttt{jlev@google.com} &\texttt{jpennin@google.com}
\end{tabular}
}

\begin{document}

\maketitle

\begin{abstract}
One of the distinguishing characteristics of modern deep learning systems is that they typically employ neural network architectures that utilize enormous numbers of parameters, often in the millions and sometimes even in the billions. While this paradigm has inspired significant research on the properties of large networks, relatively little work has been devoted to the fact that these networks are often used to model large complex datasets, which may themselves contain millions or even billions of constraints. In this work, we focus on this high-dimensional regime in which both the dataset size and the number of features tend to infinity. We analyze the performance of random feature regression with features $F=f(WX+B)$ for a random weight matrix $W$ and random bias vector $B$, obtaining exact formulae for the asymptotic training and test errors for data generated by a linear teacher model. The role of the bias can be understood as parameterizing a distribution over activation functions, and our analysis directly generalizes to such distributions, even those not expressible with a traditional additive bias. Intriguingly, we find that a mixture of nonlinearities can improve both the training and test errors over the best single nonlinearity, suggesting that mixtures of nonlinearities might be useful for approximate kernel methods or neural network architecture design.
\end{abstract}

\section{Introduction}
It is undeniable that in recent years deep learning systems have found widespread success in their applications to a diverse and ever-expanding set of domains. The foundational results on many tasks such as image recognition~\citep{krizhevsky2012}, speech recognition~\citep{hinton2012}, and machine translation~\citep{wu2016}, have begun to make their way into higher-level products that people interact with and rely upon in their daily lives. Whether these products generate a medical diagnosis, a navigation decision, or some other important output, it is crucial to understand the inner-workings of the algorithms that generate them. 

Unfortunately, our theoretical understanding of these deep learning algorithms continues to lag behind their impressive practical successes. One main challenge in building a fuller understanding stems from the fact that deep neural networks are complex nonlinear functions that employ millions or even billions~\citep{shazeer2017} of parameters. Traditional wisdom would suggest that to this parameter complexity corresponds to an optimization difficulty. Recent work, however, suggests that as the width of a network's hidden layers becomes large, the loss function simplifies and a theoretical analysis becomes tractable~\citep{Jacot:2018:NTK:3327757.3327948, mei2018mean, chizat2018global, mei2019mean, rotskoff2019global, rotskoff2018neural}. In some scenarios, the simplification is such that throughout training the parameters of the model stay within an infinitesimal radius of their initial (random) values, implying that much about neural network training can be understood by studying the random initialization \citep{Jacot:2018:NTK:3327757.3327948, chizat2018note, lee2019wide}.

Another main challenge in building a rich understanding of deep learning systems stems from the fact that they are often trained on very large, complex datasets: even if the models themselves are very large, they may not be large in comparison to the number of constraints they are designed to satisfy. Indeed, many important phenomena may become apparent only by examining the high-dimensional regime in which the dataset size and width are both large and of the same order.

In this work, we focus on the high-dimensional regime and analyze the performance of a regression model trained on the random features $F=f(WX+B)$ for a random weight matrix $W$ and random bias vector $B$. We obtain an exact formula for the training error on a noisy autoencoding task and for the test error from fitting data labeled by a linear model in the limit that the width and dataset size both go to infinity. These results are determined by the resolvent of the kernel matrix $F^\top F$, whose properties we analyze via the \emph{resolvent method} from random matrix theory. Our analysis also provides an exact formula for the eigenvalue density of the kernel matrix, which may be of independent interest since it provides a characterization for how spectral properties of the data covariance matrix propagate through neural network layers at initialization.

\subsection{Our contributions}

The main contribution of our work is an exact characterization of the training error of a ridge-regularized random feature regression model on a noisy autoencoder task and the test error when the data are labeled by a linear-teacher model in the high-dimensional regime. This is one of the first non-trivial models to be solved exactly in the joint limit of large data and large width and provides an interesting testing ground in which to analyze this regime. Some additional contributions include,
\begin{itemize}
\item An exact characterization of the spectral density of the random feature matrix ${F=f(WX+B)}$, extending prior results of \citep{pennington2017nonlinear} to non-Gaussian data distributions and to non-zero bias distributions.
\item One interpretation of the random additive bias is that it induces a distribution of activation functions parameterized by B, \emph{i.e.} ${f(Z; B):=f(Z+B)}$. Our analysis trivially extends to \emph{any} distribution of activation functions $f(\cdot\;;B)$ parameterized by $B$.
\item We show that there exists a non-trivial distribution over activation functions that outperforms the best possible single activation function in terms of memorization capacity and test error when fitting data labeled by a linear-teacher model.
\item Our method of proof introduces a surrogate ``linearization" of $F$, $F^{\text{lin}}$, that possesses the same spectral information as $F$. $F^{\text{lin}}$ and its properties are likely to be of further interest and utility in analyzing neural networks in high dimensions.
\end{itemize}

\subsection{Related Work}
Neural networks have been studied from the perspective of high-dimensional statistics in a number of recent works. Most prior work has focused on the bias-free case. The paper \cite{pennington2017nonlinear} studied the spectrum of the activation matrix $f(WX)$ for iid Gaussian data and derived an analytic expression for the training error of a ridge-regularized random feature model trained on pure noise. It is natural to consider incorporating biases by appending a constant feature 1 to $X$. Unfortunately, this leads to biases that are the same order as the weights, and so the effect disappears in the large dataset limit. Moreover, this modification on the data violates the assumptions of \cite{pennington2017nonlinear}. The paper \citep{hastie2019surprises} studies ridgeless interpolation in high-dimensional for linear features as well as nonlinear random features of iid Gaussian data. In \citep{louart2018random}, a deterministic equivalent for the resolvent of the kernel matrix $F^\top F$ is derived, which allowed for a characterization of the asymptotic training and test performance of linear ridge regression of random feature models.

Other work has investigated learning dynamics and generalization in the high-dimensional regime~\citep{liao2018dynamics, lampinen2018analytic, advani2017high} as well as the spectra of more complicated objects such as the Hessian~\citep{pennington2017geometry} and Fisher information matrix~\citep{pennington2018spectrum}. From the mathematical perspective, random matrix theory provides natural tools~\citep{Silverstein95onthe} for analyzing the behavior of neural networks in the high-dimensional regime. The paper \citep{liao:hal-01954933} examined spectra for data drawn from Gaussian mixture models; see also \citep{elkaroui} on the spectra of random kernel matrices.

Since an initial preprint of this paper was made available online, several papers have studied the generalization properties of random feature regression (all without a bias vector) in the high-dimensional limit \citep{mei2019generalization,adlam2020neural,adlam2020understanding,d2020double}. Indeed the equivalence of the spectra of the nonlinear random feature matrix $F$ to a "linearized" version of $F$ has become a crucial step in analyzing the high-dimensional asymptotics of random feature methods.

\section{Preliminaries}
\label{sec_pre}
Consider a dataset $X \in \mathbb{R}^{n_0 \times m}$ and the random feature matrix,
\[F = f(WX; B),\]
generated by a single hidden-layer network with iid Gaussian weights $W \in \mathbb{R}^{n_1 \times n_0}$ (${W_{ak} \sim \mathcal{N}(0, \sigma_W^2/n_0)}$), activation function $f$, and biases $B = b{{\bf 1}_m}^\top  \in \mathbb{R}^{n_1\times m}$ (for $b\in \mathbb{R}^{n_1}$). We regard the second argument of $f$ as parametrizing (continuously or discretely) an ensemble of activation functions. We refer to $B$ (or $b$) as the \emph{bias}, in reference to the important special case $f(WX+B)$ (additive bias). In general, however, we only assume the parameters $b_a$ are drawn iid from some distribution $\mu_B$, effecting a distribution over activation functions.

We assume that $f(\cdot ;b)$ is differential almost everywhere and $\E \absa{f(N;b)}^k$ for $N\sim{\mathcal N}(0, \sigma)$ is finite for all $1\leq k\leq 4$, $\sigma>0$, and $b\in\text{support}(\mu_B)$. When $\mu_B$ is a single Dirac mass at location $b_0$, the activation function can be written as $f(WX;B) = f(WX;b_0) = g(WX)$ for some single-argument function $g$ (for an additive bias, $g(WX) = f(WX + b_0)$). When this is the case, we say the model has a \emph{single activation function}, as opposed to a mixture or distribution of activation functions.

The quantities of interest for our investigation is the kernel matrix, $\tfrac{1}{n_1}F^\top  F$, and its \emph{resolvent},
\begin{equation}
    G(z) = \pa{\tfrac{1}{n_1} F^\top F - z I}^{-1}\,.
\end{equation}
As we review in Sec.~\ref{sec:ridge}, the optimal regression coefficients of a linear model on the random features $F$ is a simple function of this resolvent.

The high-dimensional regime that we study is the one in which the dataset size $m$, feature dimensionality $n_0$, and hidden layer width $n_1$ all go to infinity at the same rate. In particular, as is standard in the random matrix literature, we assume that we can parameterize the limit in terms of the dataset size $m$ in such a way that there exist two positive constants,
\begin{equation}
\phi := \lim_{m\to\infty} \frac{n_0(m)}{m} \quad \text{and} \quad \psi := \lim_{m\to\infty} \frac{n_0(m)}{n_1(m)}\,.
\end{equation}
Note that the resolvent is a random matrix, but as $m$ grows large, its normalized trace becomes a deterministic quantity. In the limit that $m\to\infty$, this quantity is known as the \emph{Stieltjes transform},
\eq{ \label{eq: stieltjes}
s(z) := \lim_{m \to \infty} \tfrac{1}{m} \tr\, G(z)\, .
}
Together with an auxiliary transform $\tilde{s}(z)$, defined below, these deterministic quantities completely characterize the asymptotic training error of kernel ridge regression on a noisy autoencoder task in this high-dimensional regime.

The Stieltjes transform frequently arises in random matrix methods as a way to encode the spectra of matrices. In particular, if $\lambda_i$ are the eigenvalues of $\tfrac{1}{n_1}F^\top  F$ and the empirical distribution of eigenvalues converges in distribution to some deterministic limiting density as $m\to\infty$,
\begin{equation}
\frac{1}{m} \sum_{i=1}^m \delta_{\lambda_i} \to \mu(\lambda)d\lambda\,, 
\end{equation}
then (with appropriate technical assumptions), the limiting spectral density itself can be recovered from the Stieltjes transform $s(z)$ via the inversion formula,
\eq{ \label{eq: stieltjes inversion}
    \mu(\lambda) = \lim_{\epsilon \to 0+} \frac{s(\lambda - i \epsilon) - s(\lambda + i \epsilon)}{2 \pi i \epsilon}\,.
}
The Stieltjes transform then substitutes convergence in distribution for pointwise convergence for all $z$ such that $\Im z>0$.

\subsection{Methods for computing the Stieltjes transform}
We briefly review two standard methods for computing the Stieltjes transform $s(z)$, the \emph{resolvent method} and the \emph{moments method}.

The resolvent method is an approach for computing the Stieltjes transform based on the application of the Schur complement formula to the resolvent itself (or to a closely-related block matrix). Intuitively, as the matrix size becomes large, the minors of the matrix are similar in distribution to the larger matrix, and, moreover, the Cauchy interlacing theorem guarantees that their Stieltjes transforms are close as well. This allows for the derivation of a self-consistent equation (SCE) in which the Stieltjes transform appears on the left-hand side as the trace of the resolvent, and on the right-hand side as the trace of one of its minors.

The moments method is more combinatorial in nature and involves expanding the resolvent for large $z$ and computing the traces of each term,
\eq{
s(z) = \lim_{m \to \infty} \tfrac{1}{m} \tr\, G(z) = -\lim_{m \to \infty} \sum_{k=0}^{\infty} \frac{1}{n_1^k} \frac{\tr (F^\top F)^k}{z^{k+1}}\,.
}
The traces themselves are expended out as
\eq{
\tr (F^\top F)^k = \sum F_{a_1 \alpha_1} F_{a_1 \alpha_2} \cdots F_{a_k \alpha_1}\,,
}
where the sum runs over matrix indices $a_1, \ldots, a_k, \alpha_1, \ldots \alpha_k$. The essence of the moment method involves analyzing the asymptotic contribution of each term in the sum based on its combinatorial type and the details of $F$, and resumming the results to obtain $s(z)$. 

We refer the reader to~\citep{erdos2017dynamical, TaoBook} for more details about these methods and additional background on random matrix theory.

\begin{figure}
    \centering
    \includegraphics[width=0.4\linewidth]{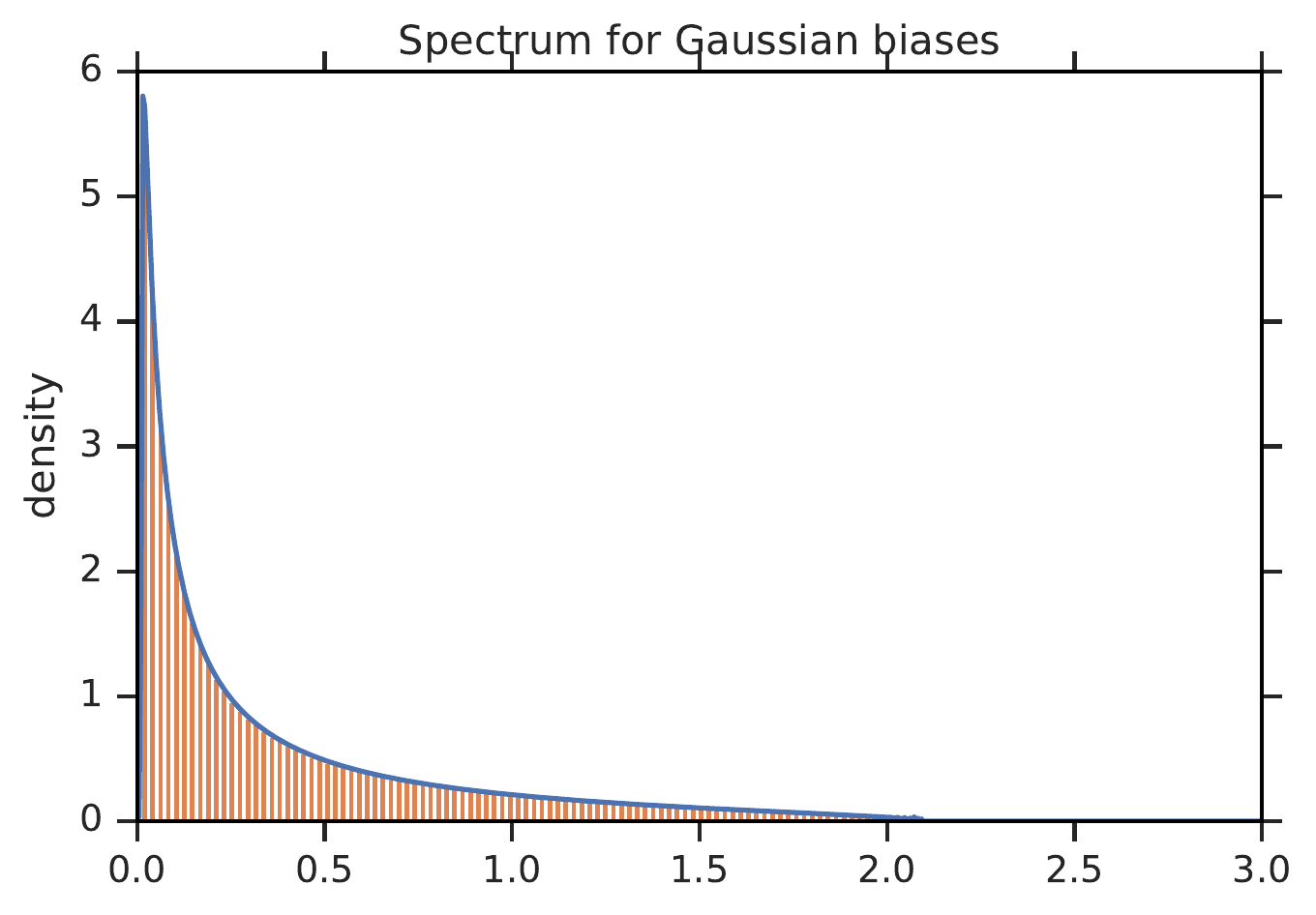}
    \quad\includegraphics[width=0.4\linewidth]{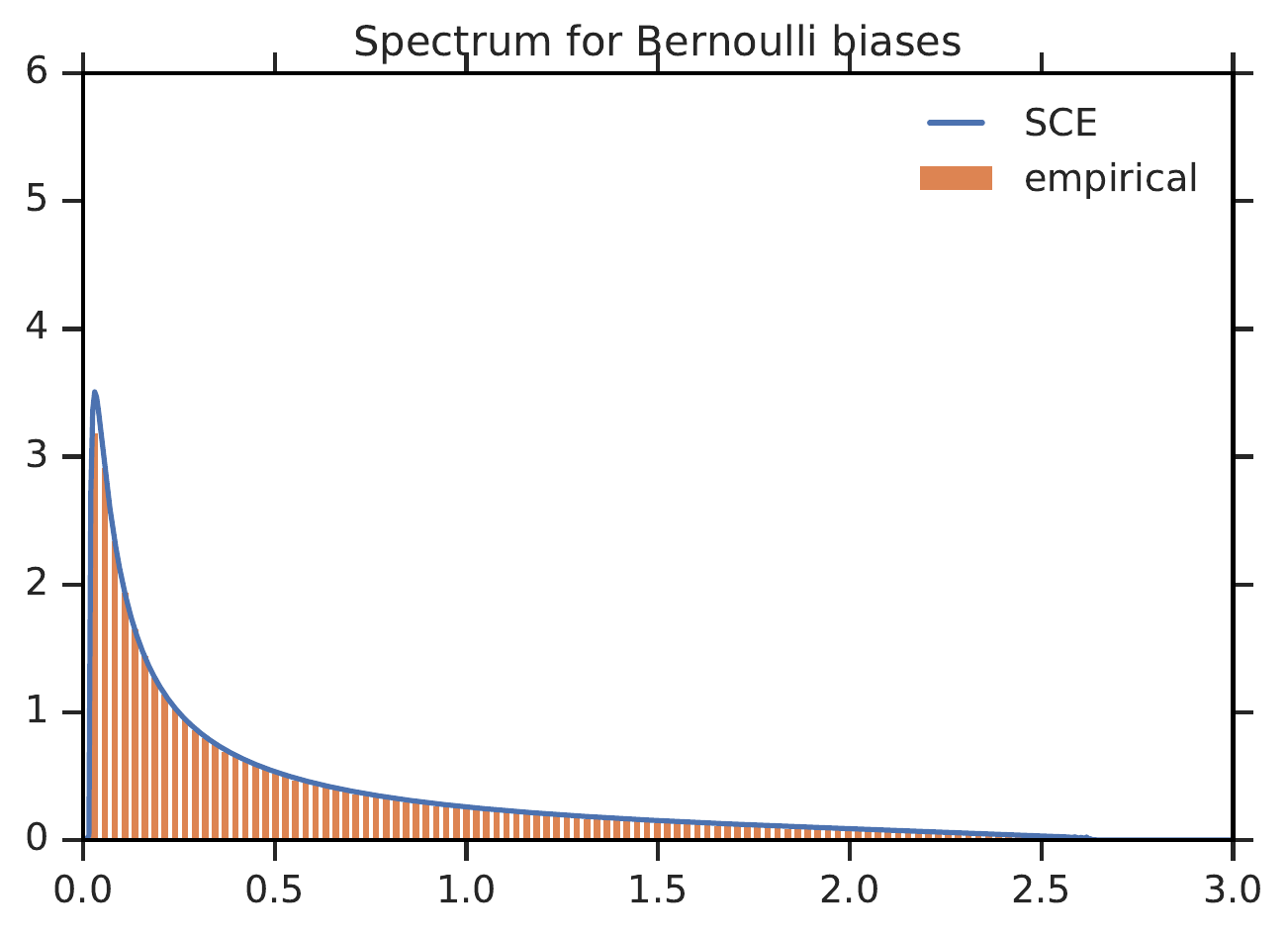}
    \caption{We get excellent agreement of theory and simulation for spectral densities for any bias distribution. We set $\phi=1.5$, $\psi=0.8$, $\sigma_X=\sigma_W=1$, and $f=\relu$. Simulations are performed on matrices of size $m = 2^{14}$.  \textbf{Left} Gaussian distribution over the biases with distribution $\mathcal{N}(0,1)$. \textbf{Right} Bernoulli distribution over biases with distribution $p=0.5$.}
    \label{fig:bias_effect_spectra}
    \vspace{-0.2cm}
\end{figure}

\section{Result for Stieltjes transform}
\subsection{Main theorem}
\label{subsec_main}
We make the following assumptions on the data matrix $X$ and bias vector $b$:
\begin{enumerate}
    \item  $\absa{\frac{1}{n_0}\sum_a X_{a\alpha}X_{a\beta}-\delta_{\alpha\beta}\sigma_X^2} \leq Cn_0^{\epsilon-1/2}$ for some positive constants $C>0$ and $\epsilon<1/100$ uniformly in $\alpha$ and $\beta$ for all $n_0>N$ for some constant $N$;
    \item  the empirical eigenvalue distribution of $\frac{1}{n_0}X^\top X$ converges in distribution to a measure $\mu_X$;
    \item $B_i\sim \mu_B$ iid, so that $\frac{1}{n_1}\sum_{a=1}^{n_1}\delta_{b_a} \to \mu_B$ in distribution.
\end{enumerate}
\begin{theorem}\label{thm: sce}
Define $\sigma_Z = \sigma_W \sigma_X$ and  resolvent $G(z) = \pa{\tfrac{1}{n_1} F^\top F - z I}^{-1}$. Then under the above assumptions and for all $z$ such that $\Im z>0$, the transforms
\begin{equation}
   \tfrac{1}{m} \tr\, G(z)\quad\text{and}\quad \tfrac{1}{m} \tr \pa{ \tfrac{1}{n_0} X^\top  X  G(z)}\,,
\end{equation}
converge in probability to the unique solution, $s(z)$ and $\tilde{s}(z)$, of the Eq.~\eqref{eq: main sce m} that map $\mathbb{C}^+$ to $\mathbb{C}^+$:
\al{\label{eq: main sce m}
    s(z) &= \E_{S \sim \mu_X}\qa{
        \frac{1}{C_0(z) + S C_1(z)}
    }\\\nonumber
    \tilde{s}(z) &= \E_{S \sim \mu_X}\qa{
        \frac{S}{C_0(z) + S C_1(z)}
    }
}
where,
\al{
\label{eq: main sce aux}
    C_0(z) &:= -z + \E_{B\sim \mu_B} \qa{ \frac{\eta(B) - \zeta(B)}{D(B)} },\\\nonumber
   C_1(z) &:= \E_{B\sim \mu_B} \qa{ \frac{\zeta(B)}{D(B)}}\,,\\\nonumber
 D(b) &:= 1+\frac{\psi}{\phi} \pa{ \zeta(b)\tilde{s}(z) +\pa{\eta(b)-\zeta(b)}s(z) } \,,
}
and where $\eta(b)$ and $\zeta(b)$ are the Gaussian expectations
\al{
    \label{eqn:eta_zeta}
    \eta(b) &:= \E_{N \sim \mathcal{N}(0, \sigma_Z^2)}\qa{f(N; b)^2} \\\nonumber
    \zeta(b) &:= \pa{\E_{N \sim \mathcal{N}(0, \sigma_Z^2)}\qa{N f(N; b) / \sigma_Z}}^2.
}
\end{theorem}
Sometimes $\zeta(b)$ is written using the (potentially weak) derivative of $f$. Using Stein's lemma or integration by parts, we see $\E\qa{N f(N; b) / \sigma_Z} = \E\qa{\sigma_Z f'(N; b)}$.

\begin{figure*}
    \centering
    \begin{tabular}{cc}
    \includegraphics[width=0.35\linewidth]{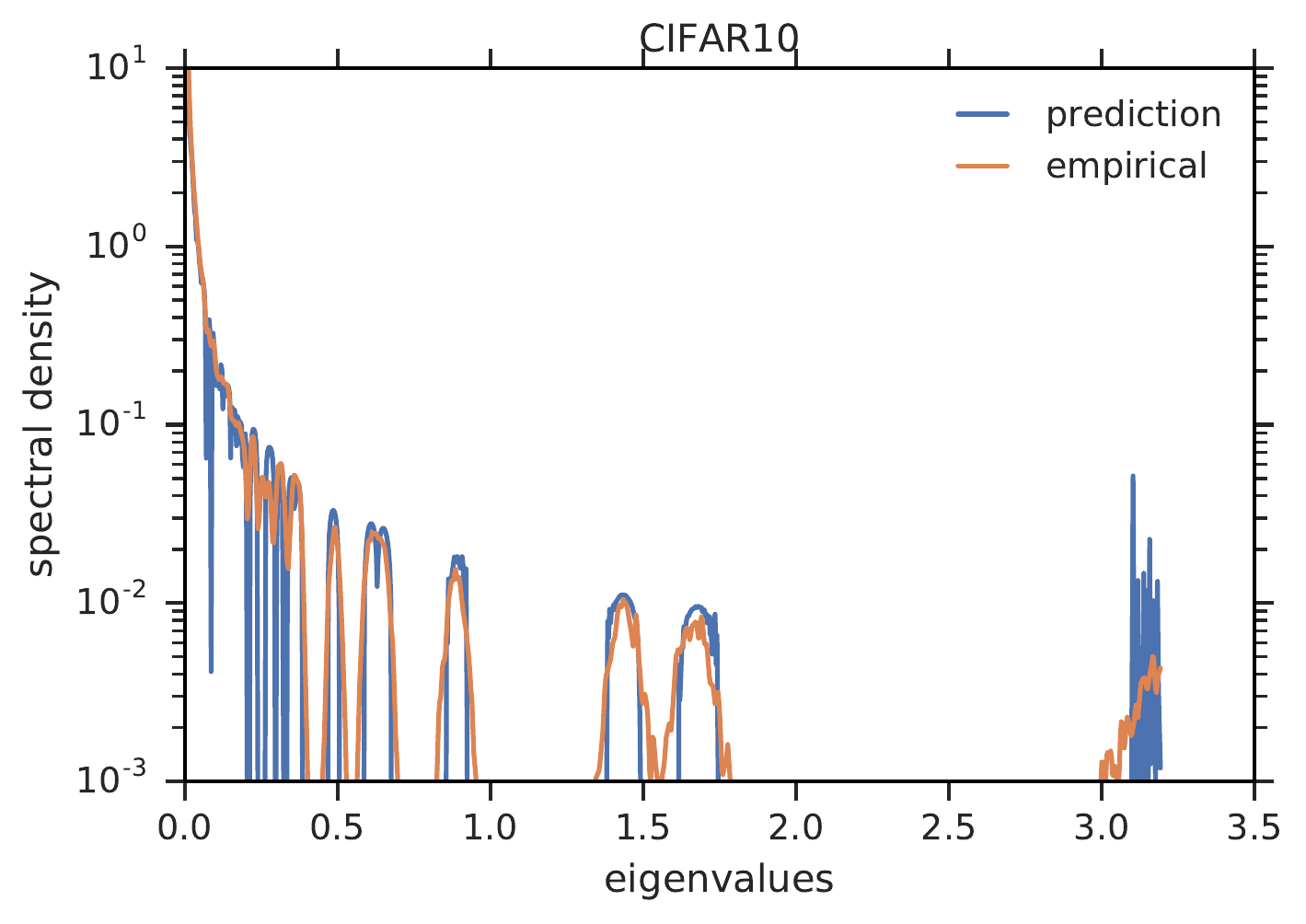} &
    \includegraphics[width=0.35\linewidth]{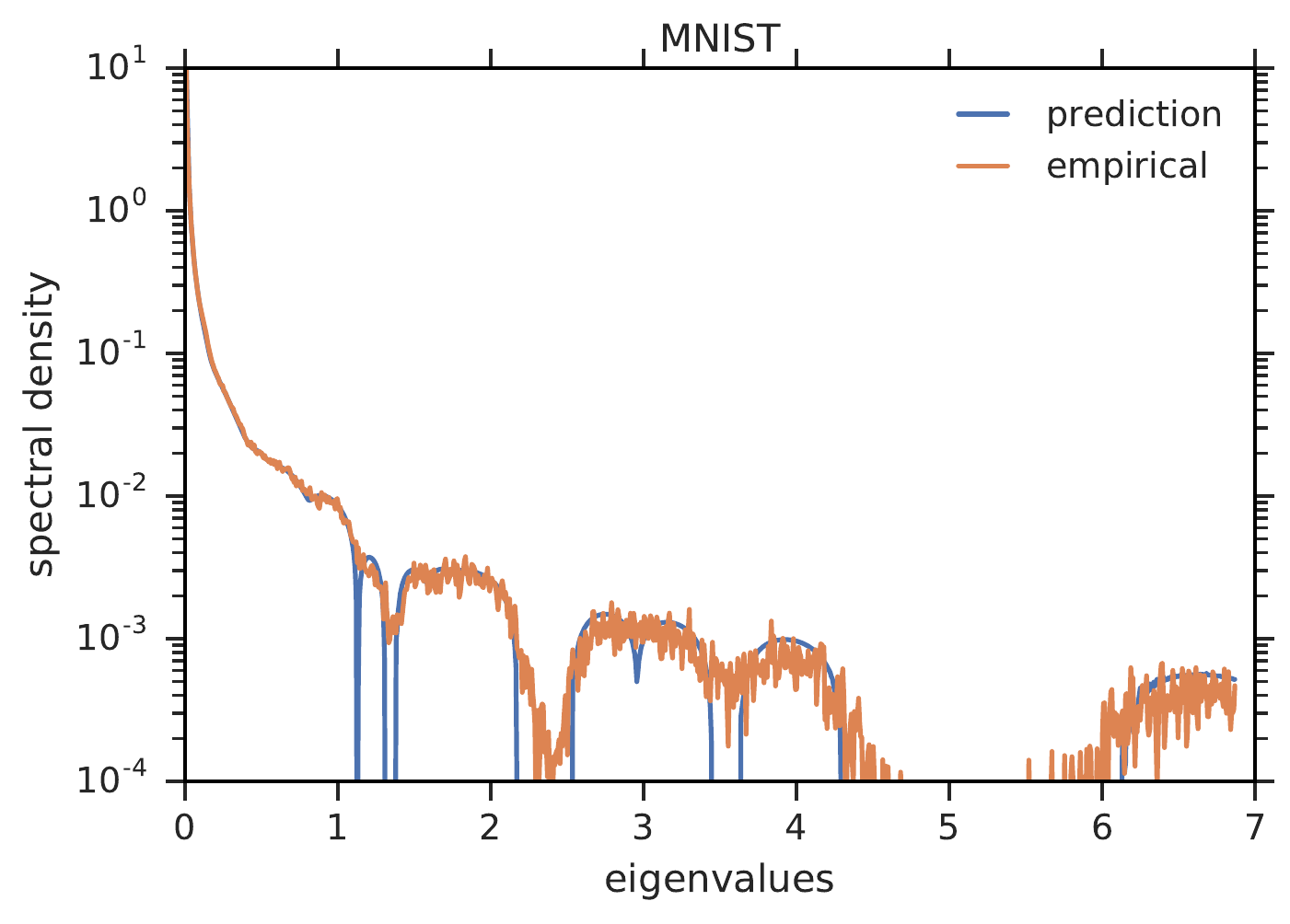} \\
    \includegraphics[width=0.34\linewidth]{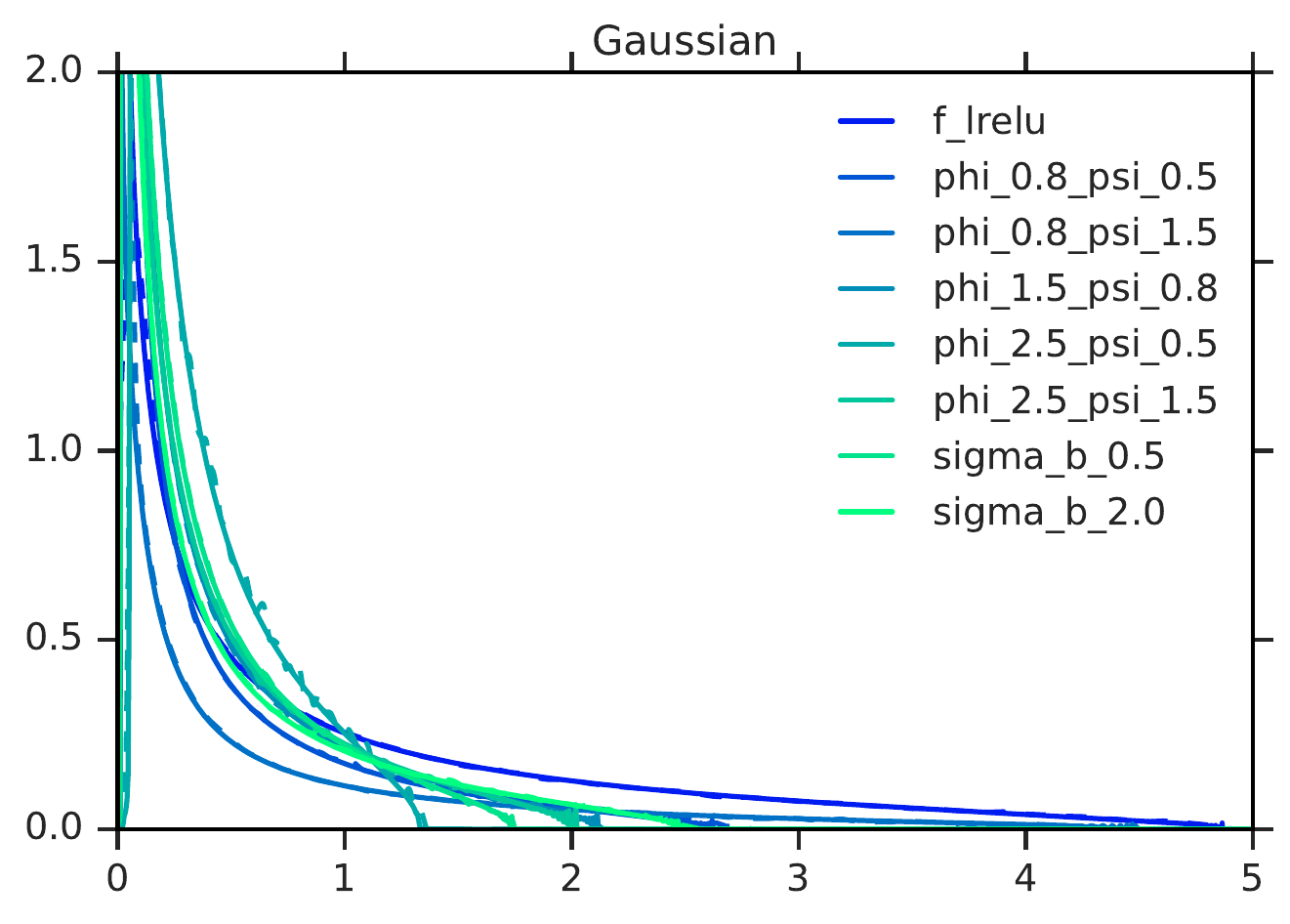} &
    \includegraphics[width=0.34\linewidth]{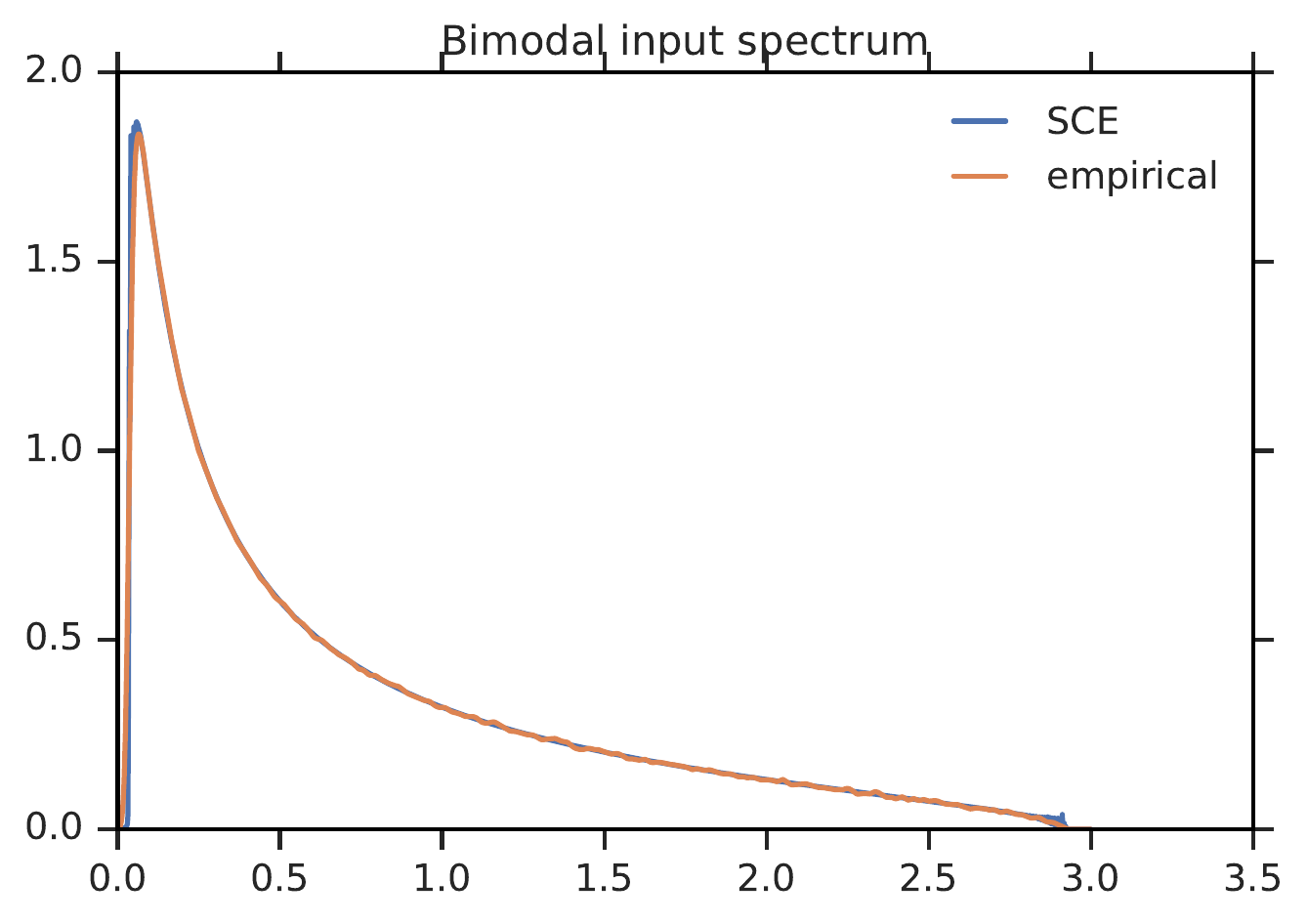}
    \end{tabular}
    \caption{Empirical spectral densities agree with our predictions for varied data distributions and shape parameters. {\bf Top left}: One class from CIFAR (airplane), mean subtracted. {\bf Top right}: Classes $\{0, 8\}$ from MNIST, mean subtracted. {\bf Bottom left}: Gaussian input data, varying the NN parameter settings and activation function. {\bf Bottom right}: Input data with a bimodal spectrum. All plots used $f=\relu, \phi=1.5, \psi=0.8$ and $\sigma_W = \sigma_B = 1$, except for the indicated modified parameter. Empirical densities were smoothed using a Gaussian KDE.}
    \label{fig:spectra}
\end{figure*}

The proof is quite involved and is presented in the SM. The basic idea is to derive a multivariate Gaussian random matrix model with the same correlation structure as $F$, then derive a self-consistent equation (SCE) using the resolvent method for this \emph{linearized} version of $F$, $F^{\text{lin}}$.

\begin{remark}
The self-consistent equations consist of two \emph{coupled} equations involving the Stieltjes transform $s(z)$ and an auxiliary object $\tilde{s}(z)$, (cf. \citep[Eq. (2)]{paul2009no}), which we will see in Cor.~\ref{cor: ridge regression} essentially measures the autoencoding capacity of the network.
\end{remark}
\begin{remark}
Note that the self-consistent equations contain an expectation over the limiting spectral density of the input data. While the assumptions on the data matrix $X$ in Thm.~\ref{thm: sce} are quite general, they may not be optimal. See Sec.~\ref{sec:experiments}, where we show strong agreement with empirical data from MNIST and CIFAR-10 and for a range of synthetic distributions. This suggests that the theorem may hold for even more general data distributions.
\end{remark}

\subsection{Alternate representation and limiting results}
When the data distribution is iid Gaussian, the expectations in Eq.~\eqref{eq: main sce m} can be expressed in closed form, though one must be careful to choose the correct branch of the resulting function. For simplicity and future reference, we focus on the setting where $0 < \phi \le \psi \le 1$, in which case we have the coupled algebraic equations,
\al{
    \label{eq:m_reduced}
    &s(z) = \Big(C_1 - (C_0 + C_1)\phi + \sqrt{C_1^2 +2(C_0 - C_1)C_1 \phi + (C_0 + C_1)^2\phi^2}\Big) / \pa{2C_0 C_1},\\
    \label{eq:mtilde_reduced}
    &\tilde{s}(z) = \frac{1 - C_0 s(z)}{C_1}\,.
}
When $\mu_B$ in Eq.~\eqref{eq: main sce aux} is trivial, \emph{i.e.} a single Dirac mass, then the result should reduce to the single activation function case with $F=f(WX)$, which was studied in~\citep{pennington2017nonlinear}. Indeed, writing $\eta = \mathbb{E}_B[\eta(B)]$ and $\zeta = \mathbb{E}_B[\zeta(b)]$ for such a distribution, and eliminating $\tilde{s}(z)$ from Eqs.~\eqref{eq:m_reduced} and \eqref{eq:mtilde_reduced}, we find that $s(z)$ satisfies the following quartic polynomial:
\al{
    0 & = (z^2 \zeta^2 \psi^2) s(z)^4 + (2 z \zeta^2 \psi (\psi - \phi) s(z)^3  + (\zeta^2(\psi-\phi)^2  + z \zeta \phi \psi + z \eta \phi^2 \psi) s(z)^2 \\\nonumber
    &\qquad + (\zeta \phi(\psi-\phi) + \phi^2(z \phi + \eta(\psi-\phi)) s(z) + \phi^3\,
}
which agrees with the result in~\citep{pennington2017nonlinear} upon identifying $s(z) = -(1-\phi/\psi)/z-\phi/\psi G(z)$.

\subsection{Spectral density estimates}
\label{sec:experiments}
The self-consistent equations in Thm.~\ref{thm: sce} can be solved numerically by iterating Eq.~\eqref{eq: main sce m} until convergence, using numerical integration. By utilizing the Stieltjes inversion formula, Eq.~\eqref{eq: stieltjes inversion}, we can extract predictions for the limiting eigenvalue density of $\tfrac{1}{n_1} F^\top F$. The results show close agreement with empirical spectral simulations from several interesting practical datasets and synthetic distributions, see Figs.~\ref{fig:bias_effect_spectra} and \ref{fig:spectra}.

\begin{figure*}[t]
    \centering
    \includegraphics[width=\linewidth]{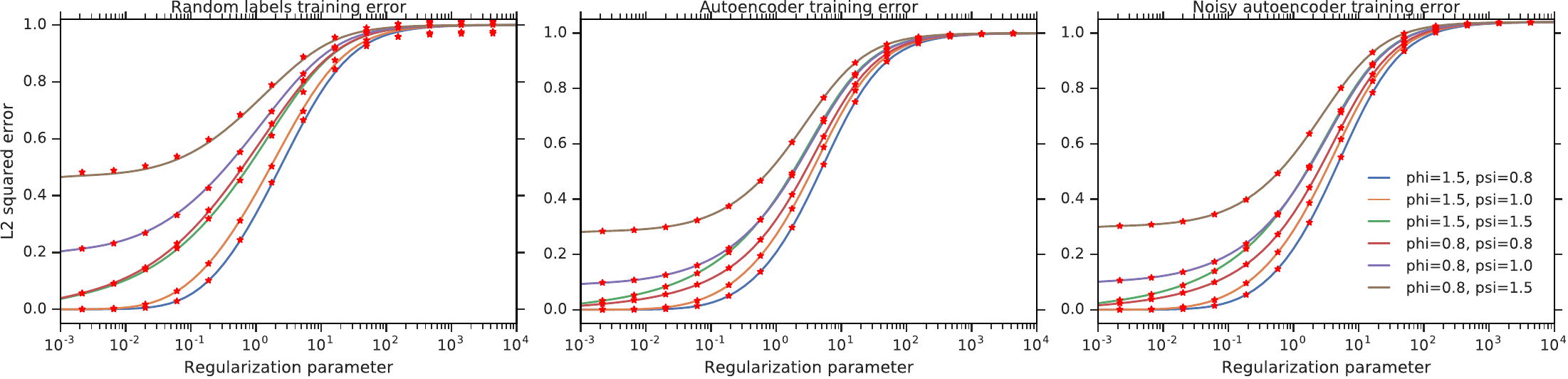} 
    \caption{Comparisons of simulated ridge regression error and our theoretical prediction. We use ReLU with $\sigma_X=\sigma_W=\sigma_B=1$ for all plots and vary the shape parameters $\phi$ and $\psi$. For simulations we use $m=2^{13}$ throughout. Note we also normalize the activation function so that $\mathbb{E}_b[\eta(b)] = 1$. \textbf{Left}: Our predictions for ridge regression with random labels are solid lines. Simulated losses are the red stars. \textbf{Center}: Autoencoder error. \textbf{Right}: Noisy autoencoder error with $\sigma_\e=0.2$.}
    \label{fig:shape_ridge_reg}
    \vspace{-0.2cm}
\end{figure*}
\section{Statistical implications}
\label{sec:ridge}

\subsection{Memorization capacity for noisy autoencoders}
First, we consider the problem of kernel ridge regression with random features given by ${F=f(WX;B)}$ and noisy regression targets given by $Y=AX + \epsilon$ for some random $A \in \mathbb{R}^{n_2 \times n_0}$ and Gaussian noise $\epsilon \in \mathbb{R}^{n_2 \times m}$ such that $\epsilon_{ij} \sim \mathcal{N}(0, \sigma^2_\epsilon)$ that are mutually independent and independent of $X$, $W$, and $B$. As is common in the literature of high-dimensional statistics, we assume an isotropic prior on $A$ such that $\E A^\top A = \tfrac{n_2}{n_0} \sigma^2_A I$. Note that when $\sigma_\epsilon=0$, we have the pure memorization setting studied by \citep{pennington2017nonlinear}. This problem provides a statistical interpretation for the role of the companion Stieltjes transform of Thm.~\ref{thm: sce}. The training loss for this problem is a measure of the memorization capacity of the model, and we study the average training loss over an ensemble of problems defined by $A$ and $\epsilon$.

\begin{corollary}\label{cor: ridge regression}
Let $W_2^*$ be the minimizer for regularized training loss
\eq{\label{eq: regularized training loss}
    \mathcal{L}(W_2) = \tfrac{1}{n_2 m} ||Y - W_2F||^2 + \tfrac{n_1}{n_2 m} \gamma ||W_2||^2\,,
}
with random features $F=f(WX;B)$. Then, with the same assumptions as Thm.~\ref{thm: sce}, the asymptotic average training error converges as
\al{ \label{eq: training error}
    E_{\text{train}} &= \E_{A,\epsilon} \tfrac{1}{n_2 m} || Y - W_2^*F||^2  \to-\gamma^2 \frac{d}{d\gamma}\big( \sigma^2_A \tilde{s}(-\gamma) +
\sigma^2_\epsilon  s(-\gamma) \big)\, .\footnote{Under stronger assumptions on $A$, the training loss consider as a random variable (without taking expectation over $A$ and $\epsilon$) can be shown to converge in probability.}
}
\end{corollary}
In particular, we see that the derivative of the Stieltjes transform $s'(-\gamma)$ measures the capacity to learn noisy labels, whereas $\tilde{s}'(-\gamma)$ measures pure autoencoding capacity. See Fig.~\ref{fig:shape_ridge_reg} for a comparison between these theoretical predictions and simulation.
\begin{proof}
The optimal weights for the regularized loss are given by
\[W_2^* = \tfrac{1}{n_1} YG(\gamma)F^\top  \quad \text{for} \quad G(\gamma) = (\tfrac{1}{n_1} F^\top F + \gamma I)^{-1},\]
resulting in training error
\al{
E_{\text{train}} &= \E_{A,\epsilon}\tfrac{1}{n_2 m} ||Y - W_2^*F||^2 \\\nonumber
&= \gamma^2 \tfrac{1}{n_2 m}\E_{A,\epsilon} \tr\pa{Y^\top  Y G(\gamma)^2} \\\nonumber
&= \gamma^2 \sigma^2_A \tfrac{1}{n_0 m} \tr\pa{X^\top  X G(\gamma)^2} + \gamma^2 \sigma^2_\epsilon \tfrac{1}{m} \tr\pa{G(\gamma)^2} \\\nonumber
&= \gamma^2 \sigma^2_A \tfrac{d}{d\gamma}\pa{ \tilde{s}(- \gamma) } + \gamma^2 \sigma^2_\epsilon \tfrac{d}{d\gamma}\pa{ s(-\gamma)} \\\nonumber
&= -\gamma^2 \big( \sigma^2_A \tilde{s}'(- \gamma) + \sigma^2_\epsilon s'(-\gamma) \big)\,. \qedhere
}
\end{proof}

\begin{figure}
    \floatbox[{\capbeside\thisfloatsetup{capbesideposition={right,top},capbesidewidth=7cm}}]{figure}[\FBwidth]
    {\caption{Performance on ridge-regularized noisy autoencoder with $\sigma_\epsilon = 1$, $\phi=1/2$, and $\psi = 1/2$. Theoretical predictions for training error (solid lines) and $1\sigma$ error bars for empirical simulations of finite networks ($n_0=192$, $n_1=384$, $m=384$) for various values of ridge regularization constant $\gamma$ as the activation function varies. \textbf{Top:} a single activation function $f$ is used. \textbf{Bottom:} the non-linearity is $f_p$, a $\text{Bernoulli}(p)$-mixture of a purely linear ($\zeta=1$) and purely non-linear ($\zeta=0$) function. Each simulation uses a randomly-chosen non-linearity having the specified values of $\zeta$, demonstrating that $E_{\text{train}}$ depends on the non-linearity solely through this constant. Red and blue stars denote minima.}\label{fig:mixture}}
    {\includegraphics[width=5cm]{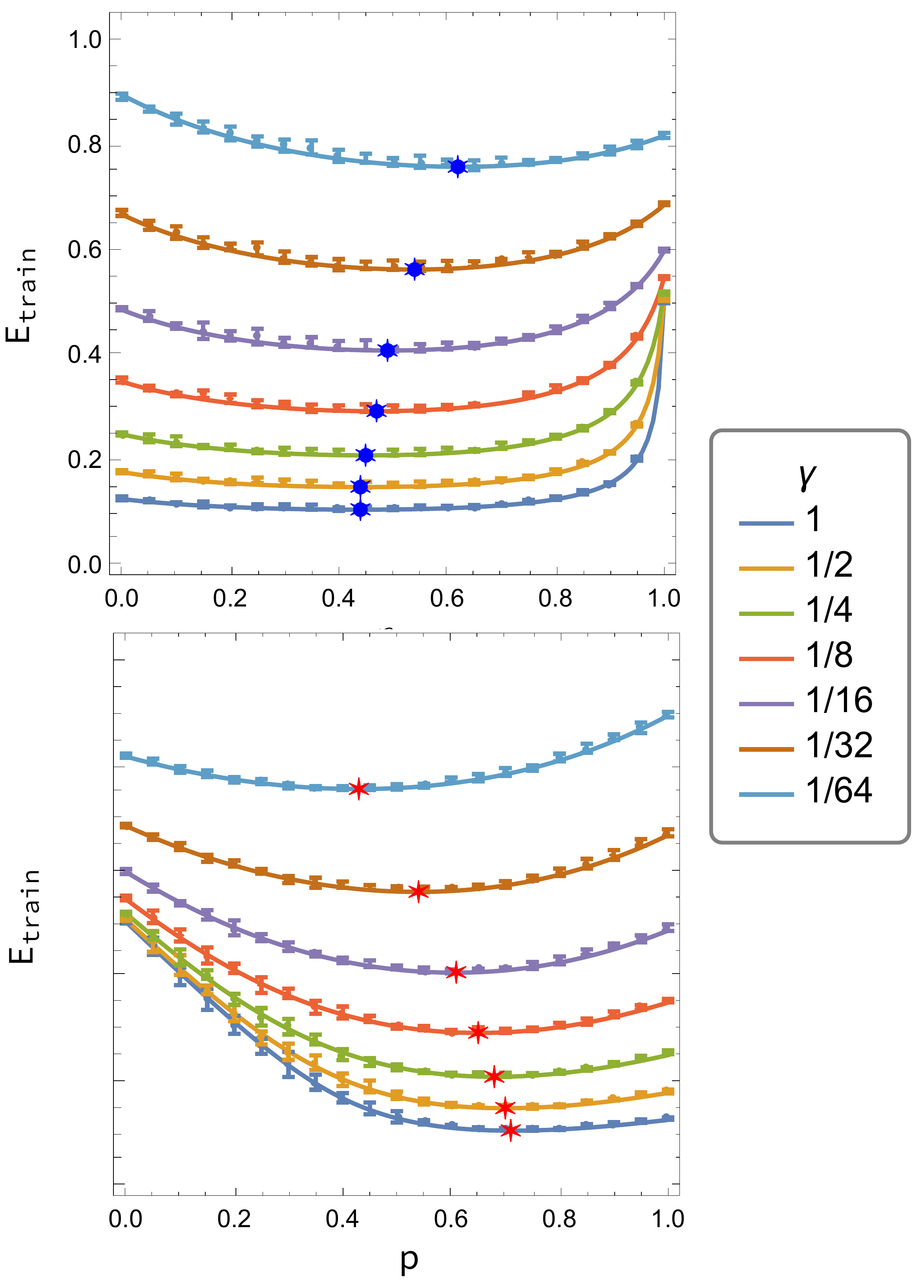}}
\end{figure}

\begin{remark}
As in~\citep{pennington2017nonlinear}, there is a scaling homogeneity in the $E_{\text{train}}$: an increase in the regularization constant $\gamma$ can be compensated by a decrease in scale of $W_2$, which, in turn, can be compensated by increasing the scale of $F$, which is equivalent to increasing $\eta(b)$ and $\zeta(b)$. Owing to this homogeneity, we are free to choose a normalization of the activation function for which $\mathbb{E}_{b\sim \mu_B}[\eta(b)] = 1$.
\end{remark}

\subsection{Nonlinear mixtures can generalize better than single nonlinearities}\label{sec_test_error}
The bias term in our random feature model can be viewed as one way of defining a distribution over activation functions. The choice of distribution, in general, affects the performance of the model on a given task---as quantified by the expectations in Eq.~\eqref{eq: main sce aux}. A key benefit of our model and analytical approach is that it permits nontrivial distributions of nonlinear activation functions.

Using Cor.\ref{cor: ridge regression}, one can show that mixtures of activation function can improve the model's capacity over the best single nonlinearity (see Sec.~\ref{sec_training_mixtures_better}). Fascinatingly, this benefit of mixtures is not restricted to the training loss but extends to the test loss. Specifically, we define a data generating process where $\mathbf{x}_i$ are $m$ iid samples from $\mathcal{N}(0, I_{n_0})$, and we label these points using a linear teacher such that $y_i = \beta^\top \mathbf{x}_i + \epsilon_i$, where $\beta_i$ and $\epsilon_i$ are iid Gaussians of zero mean and variances $1/\sqrt{n_0}$ and $\sigma_\epsilon^2$ respectively. We collect the data points $\mathbf{x}_i$ into the columns of a data matrix $X$ and labels $y_i$ into $Y$. Interestingly, in this high-dimensional limit the conditional distribution of a linear labeling function is asymptotically equivalent to a wide class of non-linear teacher neural networks (see \citep{mei2019generalization, adlam2020neural} for more details). Under our assumption on $X$, standard results from random matrix theory imply convergence of the spectrum of $X^\top X/n_0$ in distribution,  \cite{MarchenkoPastur,silverstein1995empirical}. Moreover, Assumption 1. in Sec.~\ref{subsec_main} is easily verified to hold with for probability converging to 1.

Without loss of generality, we may assume that $\E\eta(b) = 1$, since any rescaling of the activation functions $f(\cdot; b)$ can be absorbed into $W_2^*$.

The average test loss for this model is defined as
\eq{
    E_\text{test} = \E_{\beta,\epsilon} \E_\mathbf{x} \pa{ \beta^\top \mathbf{x} - W_2^* f(W \mathbf{x} ; B) }^2
}

Using the results of Thm.~\ref{thm: sce} and additional calculations (that we defer to the appendix, see \ref{sec_test_error_sm}), the test error can be characterized analytically.

Focusing first on the case of a single nonlinearity, \emph{i.e.} when $\mu_B$ is a delta mass at 0, we can use the expression for the test error to find the optimal hyperparameters to minimize the test error. Unsurprisingly given the data generating process, optimal performance is achieved by effectively performing linear regression with regularization to match the SNR. This can be obtained from the random feature model in the limit that $n_1\to\infty$ (so $\psi\to0$), which has the effect of removing the randomness from $W$ in the random feature kernel. More interestingly, the regularization in the model can be achieved with either the ridge parameter $\gamma$ or the activation function. Any configuration satisfying
\eq{\label{eq_optimal_reg}
    \frac{\gamma - \eta(0) + \zeta(0)}{\zeta(0)} = \sigma_\epsilon^2
}
is optimal. When the activation is linear, $\zeta(0)=\eta(0)=1$, and so $\gamma_\text{opt} = \sigma_\epsilon^2$, just as in a well-specified linear regression model. 

However for a fixed nonlinear activation function, $1=\eta(0)>\zeta(0)$, so Eq.~\eqref{eq_optimal_reg} implies a smaller $\gamma$ is optimal---suggesting an implicit regularizing effect of the nonlineaity. Strangely, the optimal $\gamma$ can sometimes be negative in Eq.~\eqref{eq_optimal_reg} to counteract over-regularization by the nonlinearity. Taking the opposite perspective and fixing $\gamma$, optimal performance is not always possible, as Eq.~\eqref{eq_optimal_reg} is not necessarily satisfiable since we require $\zeta\in[0, 1]$. That said, there are situations where appropriately choosing the nonlinearity can completely compensate for suboptimal $\gamma$.

When $\psi>0$, while the nonlinearity can still reduce degradation in performance, a single activation function is no longer able to completely compensate for suboptimal $\gamma$. In such situations a mixture of activation functions can help compensate further. The goal here is not to identify ``good" mixtures, since this will clearly be a dataset- and architecture-dependent question. Instead, we merely seek to demonstrate a proof-of-principle, namely that there exist non-trivial distributions over nonlinearities that can \emph{provably} outperform the best possible single nonlinearity. For this analysis, we consider a simple but nontrivial distribution over activation functions: a Bernoulli mixture of two different functions. In more detail, $\mu_B = \text{Bernoulli(p)}$ and 
\eq{\label{eq_mixture}
    f(z;b):= \begin{cases}
        \frac{x}{\sqrt{2-2p}} & \text{if } b = 0\\
        g(x) & \text{if } b = 1
    \end{cases},
}
where $g$ is such that $\eta(1) = \E g(N)^2 = 1/2p$. Note that this implies $\E \eta(B)=1$. We will optimize over $\zeta(1) = (\E \sigma_Z f'(N))^2$ and $p$.

To give a concrete example, we set $\phi=1/2$, $\psi=1/8$, $\sigma_\epsilon^2=2$, and $\gamma=2/10$. In this setting, a single activation function is unable to achieve optimal performance (see Fig.~\ref{fig_mixtures}). Specifically, the optimal test error of approximately 0.945 is achieved at $\zeta\approx 0.634$, which falls significantly short of the optimal test error of approximately 0.581 for $\gamma\approx 2.0$ and a linear activation function. 

The mixture model from Eq.~\eqref{eq_mixture} can significantly reduce this gap. By performing a simple grid search over $p$ and $\zeta(1)$, we find that $p=0.999$ and $\zeta(1)\approx0.364$ yields a test error of approximately 0.783. We illustrate these conclusions in Fig.~\ref{fig_mixtures}, and provide empirical results from simulations confirming our findings. 

\begin{figure*}
    \centering
    \begin{tabular}{cc}
    \includegraphics[width=0.4\linewidth]{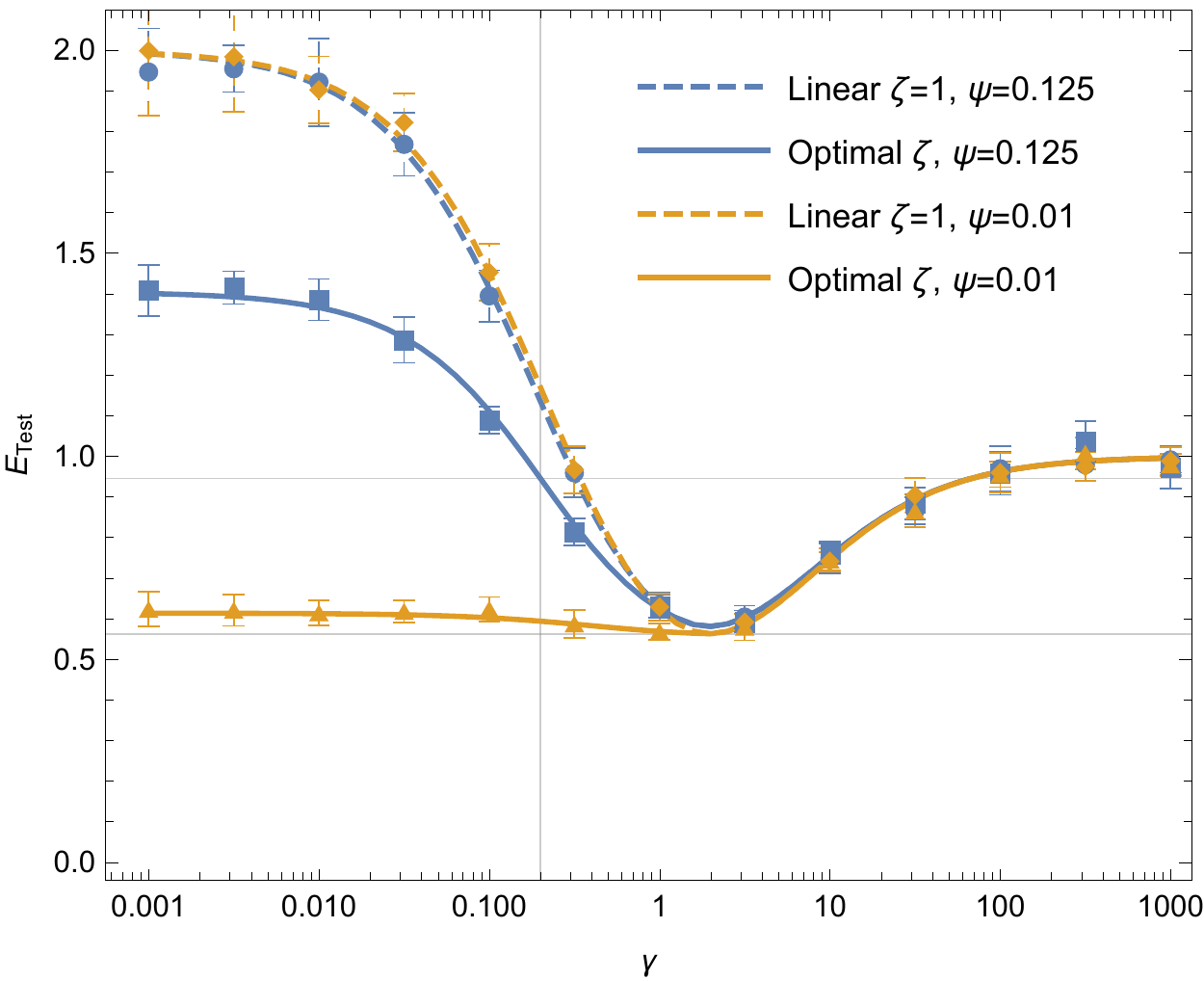} &
    \includegraphics[width=0.4\linewidth]{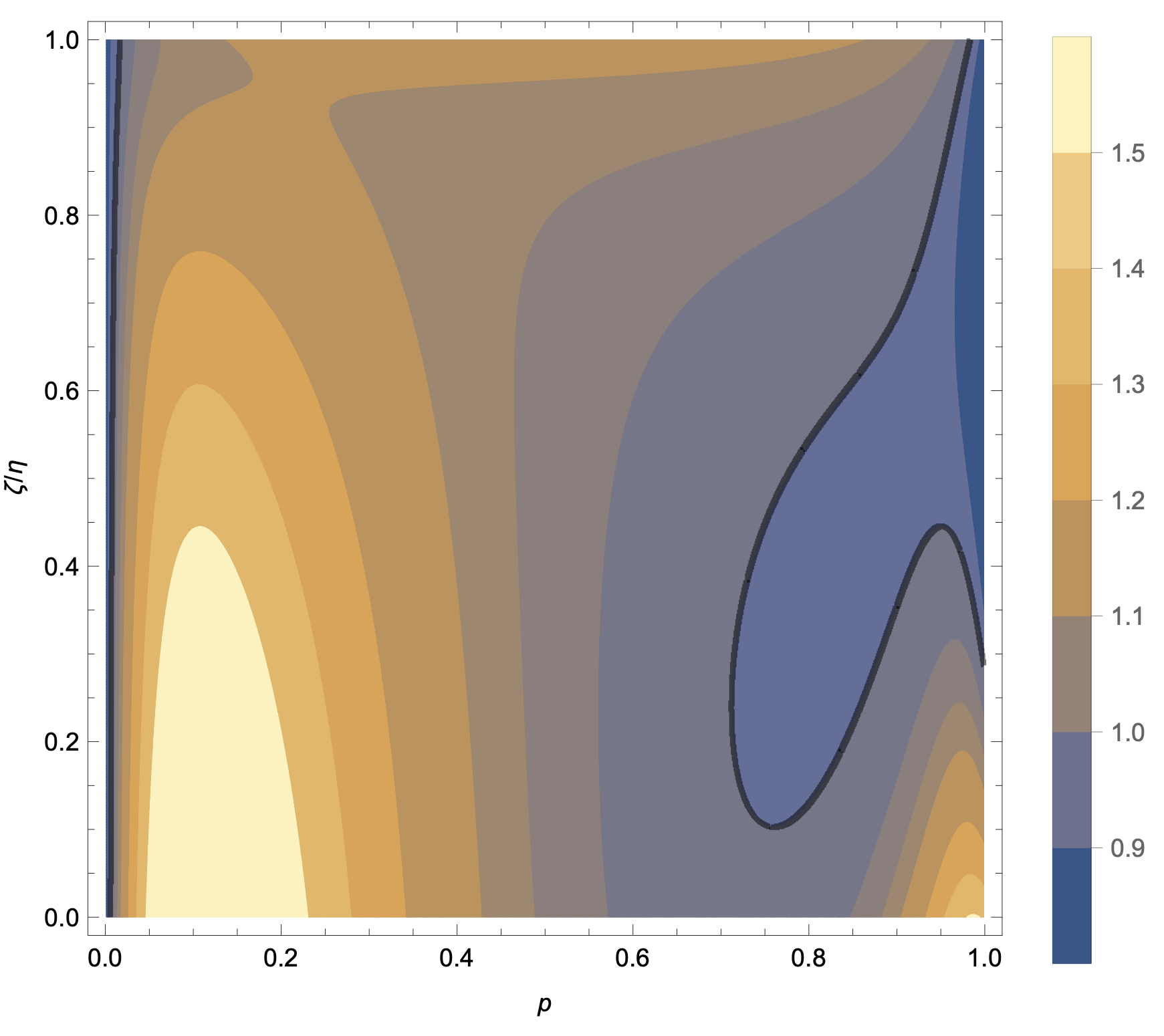}
    \end{tabular}
    \caption{We plot $E_\text{test}$ for random feature regression when the data are labeled by a linear model with $\phi=1/2$ and $\sigma_\epsilon^2=2$. \textbf{Left}: Plots of $E_\text{test}$ as a function of $\gamma$ for both the linear activation function (dashed) and an optimal nonlinear activation function (solid). When $\psi=1/100$ (orange) tuning $\zeta$ can almost overcome a too small $\gamma$ and recover optimal $E_\text{test}$. When $\psi=1/8$ (blue) tuning $\zeta$ reduces $E_\text{test}$ over $\zeta=1$, but the gap to optimal $E_\text{test}$ is large. The dots and errors bars ($\pm1$ s.d.) show simulations results for $n_0=1000$ over 10 trials. \textbf{Right}: Keeping $\psi=1/8$, we show a contour plot of $E_\text{test}$ for the binary mixture Eq.~\eqref{eq_mixture}. We outline the region where the mixture improves over all single nonlinearities with a grey line. The best setting found in our grid search is on the far right of the plot at $p=0.999$ and $\zeta(1)\approx0.364$. Simulating the binary mixture for $n_0=1000$ over 10 trials estimated $E_\text{test}$ as $0.785\,(\pm0.035)$.}
    \label{fig_mixtures}
\end{figure*}

\section{Conclusions}
In this work we studied the feature matrix $F=f(WX; B)$ where $W$ is a random matrix with iid Gaussian entries. Under mild assumptions on $X$ and $B$, we obtained an exact analytic formula, Eq.~\eqref{eq: main sce m}, that characterizes the Stieltjes transform of the spectral density of $F$. The result allowed us to describe the exact training loss of a ridge-regularized noisy autoencoder and the test error from fitting data labeled by a linear model in the high-dimensional, large-dataset limit, providing one of the first closed-form solutions to a non-trivial model in this limit. We found excellent agreement between the asymptotic predictions of Eq.~\eqref{eq: main sce m} and a variety of finite-dimensional empirical simulations.

We also advanced the interpretation of the bias $B$ as one particular way of parameterizing a distribution of activation functions. Indeed, our derivations proceed completely unchanged whether this distribution is of the traditional additive form $f(\cdot + B)$ or the more general $f(\cdot; B)$. By examining the latter, we showed that there are configurations in which a non-trivial distribution over activation functions provably outperforms the best possible single activation function. This opens the door to future investigations regarding optimal methods for parameterizing distributions over activation functions for approximate kernel methods, and suggests the possibility that mixtures of nonlinearities could be a useful design consideration when constructing neural network architectures.

\bibliography{references}
\bibliographystyle{abbrv}


\appendix

\begin{center}

\onecolumn
\textbf{{\large Supplementary Materials:} \\
A Random Matrix Perspective on Mixtures of Nonlinearities in High Dimensions}
\end{center}
\setcounter{equation}{0}
\setcounter{figure}{0}
\setcounter{table}{0}
\setcounter{page}{1}
\setcounter{section}{0}
\setcounter{remark}{0}
\makeatletter
\renewcommand{\theequation}{S\arabic{equation}}
\renewcommand{\thefigure}{S\arabic{figure}}

\section{Outline for the proof}

The main goal of the supplementary material is to prove Thm.~\ref{thm: sce}. This is achieved in three steps. The first step is to derive the leading order correlation structure (or expected kernel) for $F$ (Lem.~\ref{lem: cor structure}), and specify a multivariate Gaussian model with the same correlation structure (Subsec.~\ref{subsec: Flin}). We refer to this Gaussian model as a \emph{linearized model}, denoted $F_\text{lin}$, since it removes the nonlinearity $f$ from the random part of our random matrix ensemble. The matrix $F$ is interesting as the data induce covariances among columns and the biases among rows, and both these must be captured in the linearized model. Random matrices with covariances of this type have been studied before when they can be written as $A^{1/2}XBXA^{1/2}/n$ for $X$ with iid entries and $A$ and $B$  nonnegative definite Hermitian \cite{lixin2007spectral,burda2005spectral,paul2009no,el2009concentration}, and it is known they lead to coupled,
functional self-consistent equations \cite{paul2014random} such as Eq.~\eqref{eq: main sce m}.

Since $F$ is not of the form $A^{1/2}XBXA^{1/2}/n$, the second step is to show that the limiting Stieltjes transform for the kernel matrix $F^\top F$ is equivalent to a multivariate Gaussian random matrix model with an identical correlation structure, \emph{i.e.} equivalent to $F_\text{lin}$. Results of this form are common in random matrix theory, where under correlation assumptions on the entries the matrix Dyson equation determines the limiting behavior of the resolvent \cite{erdos2019matrix}. Moreover, this deterministic equation only depends on the second order moments through a super operator. Thus, under some assumptions on the correlation structure, random matrices with the same second order moments have the same asymptotic spectra.

These assumptions on the correlation structure come in many forms: the papers \cite{banna2015limiting2,banna2015limiting} have very general conditions when the matrix entries can be written as functions of an iid random field. Cor.~1.1 of \cite{bai2008large} gives conditions based on the concentration of quadratic forms that imply the limiting Stieltjes transform only depends on $t_{\a\beta}:=\E F_{a\a}F_{a\beta}$.\footnote{Note this cannot depend on $a$!} While this result requires that the rows of the matrix are i.i.d., the argument has been generalized from the identically distributed setting in \cite{yaskov2014universality}. Again the condition is based on the concentration of quadratic forms, see Remark 1 and Assumptions (A3) and (A4) of \cite{yaskov2014universality}.

The more general result of \cite{yaskov2014universality} is required for the following reason: Either $B$ can be consider as random or it can be conditioned on in all expectations. If $B$ is taken to be random and $b_a$ are drawn i.i.d. from $\mu_B$, then the rows of $F$ are also i.i.d., so the setting of \cite{bai2008large} applies. However, the concentration of the quadratic forms fails for $t_{\a\beta}:=\E_B\E_W[ F_{i\a}F_{i\beta}|B]$. Instead all expectations must be considered conditional on $B$, and since $B$ is independent of $W$, it is sufficient to take $B$ deterministic such that $\frac{1}{n_1}\sum_{a=1}^{n_1}\delta_{b_a} \to \mu_B$. Defining $t_{\a\beta}^{(a)}:=\E[ F_{a\a}F_{a\beta}|B]$, one can show concentration of the required quadratic forms, but the rows of $F$ are no longer identically distributed and \cite{bai2008large} does not apply. To this end, we consider all expectations in the proof to be taken over $W$ and be conditional on $B$ and $X$, \emph{i.e.} $\E[\cdot] \equiv \E[\cdot|B, X]$.

The third a final step of the proof is a derivation of a self-consistent equation specifying the limiting Stieltjes transform of the Gaussian model. We do this using the resolvent method (see \cite{erdos2017dynamical} for an introduction). While the application of this method to multivariate Gaussian covariance matrices is standard, we note the form of the SCE is new and that we include the derivation here for completeness.

\section{Correlation structure of \texorpdfstring{$F$}{}}\label{sec: cor of F}
To begin we derive an asymptotic form for correlation structure of $F$. For simplicity of presentation, we derive the results for $f(\cdot,B)\equiv f(\cdot+B)$, but the argument generalizes directly. We recall that $W$ is an $n_1\times n_0$ random matrix with mean 0 and variance $\sigma_W^2/n_0$ iid Gaussian entries. Define ${c}_{\a\beta}:=\sum_k X_{k\a}X_{k\beta}/n_0$. Recall that $X$ is assumed to satisfy
\eq{
    c_{\a\beta}-\delta_{\a\beta}\sigma_X^2=\calO\pa{n_0^{\e-1/2}}
}
for all $\a$ and $\beta$.

Observe that
\eq{\label{eq_gaussian_dist_of_z}
    Z_{a\a}:=\sum_{k=1}^{n_0} W_{ak}X_{k\a} \text{ and } Z_{b\beta}:= \sum_{k=1}^{n_0} W_{bk}X_{k\beta}
}
are jointly Gaussian. Moreover,
\eq{\label{eq_dist_Z}
	\E Z_{a\a} = 0 \quad \text{and} \quad \E Z_{a\a}Z_{b\beta} = \sigma_W^2 c_{\a\beta}\mathbf{1}\pa{a=b}.
}
Note in particular, that $Z_{a\a}$ and $Z_{b\beta} $ are independent if $a\neq b$. For convenience, we normalize ${c}_{\a\beta}$ and define $\tilde{c}_{\a\beta}:={c}_{\a\beta}/\sigma_X^2$ and $\e_\a:=\tilde{c}_{\a\a}-1$, and note $\tilde{c}_{\a\beta} = \calO(n^{\e-1/2})$ for $\a\neq\beta$ and $\e_\a=\calO(n^{\e-1/2})$ for all $\a$.

In order to specify the correlations, we employ transforms of the activation function $f$. The transforms are Gaussian integrals of $f$ at different locations and scales. Let $N\sim\mathcal{N}(0,1)$ and recall $\sigma_Z=\sigma_X\sigma_W$, then define
\eq{
	\xi_0(x) \!:= \E\qa{ f\pa{\sigma_Z N \!+\! x}},\quad\! \xi_1(x) \!:= \E\qa{N f\pa{\sigma_Z N \!+\! x} },\quad\! \xi_2(x) \!:= \E\qa{ \frac{(\sigma_{Z}^2\!-\!1) N^2}{2\sigma_{Z}^2}f\pa{\sigma_Z N \!+\! x}}\!,
}
\eq{
	\eta(x)\!\equiv\!\eta_0(x)\!:= \E\qa{ f(\sigma_Z N \!+\! x)^2},\quad\! \eta_{2}(x)\!:= \E\qa{ \frac{(\sigma_{Z}^2\!-\!1) N^2}{2\sigma_{Z}^2}f\pa{\sigma_Z N \!+\! x}^2}\!,\!\quad\!\text{and} \!\quad\! \zeta(x) \!:= \xi_1(x)^2
}
We also need the following simple lemma that can be proved with Taylor's theorem.

\lem{\label{lem: taylor expansion f bias}
	Suppose $\e$ is a small constant (\emph{i.e.} $|\e|<1/10$) and that $N\sim \mathcal{N}(0,1)$, then
	\eq{\label{eq_mean}
		\E_N f\left(\sigma_Z  \sqrt{1 +\e} N + b \right) = \xi_0(b) +\xi_2(b) \e+\calO(\e^2)
	}
	and
	\eq{\label{eq_var}
		\E_N f\left(\sigma_Z  \sqrt{1 +\e} N + b \right)^2 =  \eta(b) + \eta_2(b)\e+\calO(\e^2)
	}
}

\begin{proof}
    A natural approach to proving Lem.~\ref{lem: taylor expansion f bias} is to Taylor expand $f$ point-wise at $\sigma_Z N + b$ in the small quantity $\sigma_Z \pa{\sqrt{1+\e}-1} N + b$, and then take expectation over $N$. However, this argument requires additional regularity assumptions on $f$. Instead, we can write
    \eq{\label{eq_integral_over_z}
        \E_N f \left(\sigma_Z  \sqrt{1 +\e} N + b \right) = \int_\mathbb{R} f(z) \phi_\e(z) dz,
    }
    where $\phi_\e$ is the p.d.f. of $\mathcal{N}(b, \sigma_Z^2\pa{1+\e})$, and then Taylor expand $\phi_\e$ in $\e$ about $0$. Note when $\e=0$, $\phi_0$ is the p.d.f. of $\mathcal{N}(b,\sigma_Z^2)$ the distribution of $\sigma_Z N$ for $N\sim\mathcal{N}(0,1)$.
    
    Note that the function $\e\mapsto \phi_\e(z)$ is $C^\infty$ in an open interval $I$ containing all values of $\e$ allowed in the lemma statement and all for all $z\in\mathbb{R}$. Moreover, this function and its derivatives are bounded uniformly in $z$ over the same open interval $I$, since we assume $\sigma_X>0$ and $\sigma_W>0$.\footnote{In general, the derivatives look like $\phi_\e(z)p_\e(z)$ for some polynomial in $p_\e$ with $\e$-dependent coefficients that are finite for $\e$ in the open interval $I$.} Thus, for all $z$, we have the Taylor expansion of $\phi_\e(z)$:
    \eq{\label{eq_taylor_phi}
        \phi_\e(z) = \phi_0(z) + \e \frac{(z-b)^2 - \sigma_Z^2}{2 \sigma_Z^2} \phi_0(z) + \e^2 R_\e(z),
    }
    where $R_\e(z)$ is a remainder term. Using the Lagrange form of the remainder, we can write\footnote{The derivative is with respect to $\e$ in the display below.}
    \eq{
        R_\e(z) = \frac{ \phi_{\e'}^{(2)}(z) }{2} \e^2 = \e^2 p_{\e'}(z) \phi_{\e
        '}(z) 
    }
    for some $\absa{\e'}\leq\e$, where $p_\e(z)$ is a degree-4 polynomial with $\e$-dependent coefficients that are finite for $\e$ in the open interval $I\ni \e'$.

    Now, we can use Eqs.~\eqref{eq_integral_over_z} and \eqref{eq_taylor_phi} to see
    \eq{\label{eq_intergral_taylor}
        \E_N f \left(\sigma_Z  \sqrt{1 \!+\!\e} N \!+\! b \right) =\E_N\qa{ f \left(\sigma_Z N \!+\! b \right)} + 
        \e\, \E_N\qa{ f \left(\sigma_Z N \!+\! b \right) \frac{(\sigma_Z^2\!-\!1) N^2}{2\sigma_Z^2} } + \int_\mathbb{R} f(z) R_\e(z) dz.
    }
    The first two term of the right-hand side of Eq.~\eqref{eq_intergral_taylor} are as expected, and for the last term we see
    \eq{
        \int_\mathbb{R} f(z) R_\e(z) dz \leq |\e|^2 \pa{\int_\R |f(z)|^2 \phi_\e(z)dz \int_\R |p_\e(z)|^2 \phi_\e(z)dz }^{1/2} =\calO(\e^2),
    }
    by assumption on $f$. This proves Eq.~\eqref{eq_mean}. An identical argument applied to $f(z)^2$ proves \eqref{eq_var}.
\end{proof}

\paragraph{Without loss of generality, $\E F_{a\a}=0$.} 

Eq.~\eqref{eq_mean} implies
\eq{
    \E F_{a\a} =\E_{N\sim \mathcal{N}(0, \sigma_{W}^2 c_{\a\a}) } [f(N + b_a)] =\xi_0(b_a) +\xi_2(b_a) \e_\a+\calO(\e_\a^2).
}
Thus, the two leading order terms of $\E[F|B]$ are both low-rank and so cannot change the spectrum or Stieltjes at order 1 (for more detail see \cite{bai2008large}), and the remainder term has Frobenius norm bounded by $\calO(n_0^{2\e})$ and so also cannot affect the spectrum or Stieltjes transform at order 1. In the notation above, we say $\xi_0(b)=0$ for all $b$.

We now derive the correlation structure of $F$ conditional on $B$ (recall $X$ is deterministic, but we may view this as conditioning on $X$).
\begin{lemma}\label{lem: cor structure}
    To leading order the correlation structure of $F$ is 
    \eq{
        \E[F_{a\a}F_{b\beta} | B]=\begin{cases}
            0   & \text{if } a\neq b \\
             \eta(b_a) & \text{if } \a=\beta \text{ and } a=b \\
            \zeta(b_a)c_{\a\beta} & \text{if } \a\neq\beta \text{ and } a=b 
        \end{cases}.
    }
\end{lemma}

\begin{proof}
Using the observation in Eq.~\eqref{eq_gaussian_dist_of_z}, we repeatedly rewrite the expectations in $F^\top F$ as we can reduce the expectation over $W$ to an expectation over at most two correlated Gaussian random variables. By independence (see Eq.~\eqref{eq_dist_Z}), $\E F_{a\a}F_{b\beta}=\E F_{a\a}\cdot\E F_{b\beta}=0$ for $a\neq b$. Using Lem.~\ref{lem: taylor expansion f bias}, we get
\eq{\label{eq: variance bias-case}
	\E f(Z_{a\a})^2 = \eta({b}_{a}) + \eta_{2}({b}_{a}) \e_\a+\calO(\e_\a^2).
}

For the covariance calculation, $\E F_{a\alpha}F_{a\beta} $, we repeat the argument in the proof of Lem.~\ref{lem: taylor expansion f bias} except we define  $\varphi_{\e_1,\e_2,\rho}(z_1,z_2) $ as the bivariate p.d.f of
\eq{
    \mathcal{N}\pa{ \begin{pmatrix} 0\\0\end{pmatrix} ,\sigma_Z^2\begin{pmatrix}  1+\e_\alpha & \tilde{c}_{\alpha\beta} \\ \tilde{c}_{\alpha\beta} & 1+\e_\beta \end{pmatrix} }.
}
We then perform a multivate Taylor expansion about $(0,0,0)$ in $(\e_\alpha,\e_\beta,\tilde{c}_{\alpha\beta})$ to first-order. Note that under $\varphi_{0,0,0}$ the coordinates are independent and $\varphi_{0,0,0}(z_1,z_2)=\phi_0(z_1)\phi_0(z_2)$, so we can factor any integral over $z_1$ and $z_2$ using Fubini's theorem. We skip the details for the remainder term, as they are similar to before, but the leading order terms are
\al{
    \E F_{a\alpha}F_{a\beta} &= \int_{\R^2} f(z_1)f(z_2) \varphi_{\e_1,\e_2,\rho}(z_1,z_2) dz_1dz_2 \\
    &= \int_{\R^2} f(z_1)f(z_2) \varphi_{0,0,0}(z_1,z_2) dz_1dz_2 \\
    &\qquad+ \e_1\int_{\R^2} \frac{(z_1-b)^2 - \sigma_Z^2}{2 \sigma_Z^2} f(z_1)f(z_2) \varphi_{0,0,0}(z_1,z_2) dz_1dz_2 \\
    &\qquad+ \e_2\int_{\R^2} \frac{(z_2-b)^2 - \sigma_Z^2}{2 \sigma_Z^2} f(z_1)f(z_2) \varphi_{0,0,0}(z_1,z_2) dz_1dz_2 \\
    &\qquad+ \tilde{c}_{\alpha\beta}\int_{\R^2} \frac{ (z_1-b)(z_2-b) }{\sigma_Z^2} f(z_1)f(z_2) \varphi_{0,0,0}(z_1,z_2) dz_1dz_2 + R,
}
where $R=\calO(n_0^{2\e-1})$.  Using Fubini's theorem and the assumption that $\int_R f(z)\phi_0(z)dz=0$, we find that 
\eq{
    \E F_{a\alpha}F_{a\beta} = \tilde{c}_{\alpha\beta} \pa{\int_{\R} \frac{z-b}{\sigma_Z} f(z)\phi_{0}(z) dz}^2 + \calO(n_0^{2\e-1}).
}

\end{proof}

\subsection{Linearized model}\label{subsec: Flin}
Define
\eq{
	\Flin := \mathcal{C} \Theta^1 \Sigma + \pa{\mathcal{V}^2-\mathcal{C}^2}^{1/2} \Theta^2,
}
where: $\Theta^1$ and $\Theta^2$ are $n_1\times m$ matrices that have iid Gaussian entries with mean 0 and variance $1/n_1$; 2) $\mathcal{V}$ and $\mathcal{C}$ are $n_1\times n_1$ diagonal matrices with entries
\eq{
	\mathcal{V}_a^2 =\eta(b_a) \quad\text{and}\quad \mathcal{C}_a^2 =\zeta(b_a),
}
for a bias vector $b$; 3) $\Sigma$ is a matrix square root of $\frac{1}{n_0}X^\top X$.

A simple calculation shows that the entries of the matrix $\Flin$ match the first and second mixed moments of $F$ given in Lem.~\ref{lem: cor structure} up to residuals of size $\calO(n_0^{2\e-1})$ in the off-diagonal entries and $\calO(n_0^{\e-1/2})$ in the diagonal entries. Therefore, the matrix of residuals has Frobenius norm at most $\calO(n_0^{2\e})$, which implies the limiting spectra of $[\E F_{1\a}F_{a\beta}]_{\a\beta}$ and $[\E F^{\text{lin}}_{1\a}F^{\text{lin}}_{a\beta}]_{\a\beta}$ agree to $\calO(1)$. However, unlike $F$, $\Flin$ is linear in the random matrices $\Theta^1$ and $\Theta^2$; in this sense, it is a linearization of $F$.

\section{Verifying Assumptions (A3) and (A4) of \texorpdfstring{\cite{yaskov2014universality}}{}}

We define
\eq{
    \Sigma^{(a)}_{\a\beta}:=\E F_{a\a} F_{a\beta}= \E_{N_\a,N_\beta}\qa{f(N_\a+b_a)f(N_\beta+b_a)}
}
for
\eq{\label{eq_mvg}
    (N_\a,N_\beta)\sim \mathcal{N}\pa{ \begin{pmatrix}0\\0\end{pmatrix} ,\sigma_W^2 \begin{pmatrix} c_{\a\a} &  c_{\a\beta} \\ c_{\beta\a} &  c_{\beta\beta} \end{pmatrix} },
}
whose leading order behavior we have understood in Lem.~\ref{lem: cor structure}. 

\paragraph{Assumption (A3).} Define 
\eq{
    \calQ:= \sum_{\a,\beta=1}^{m} F_{a\a}F_{a\beta} A_{\a\beta} - \tr\pa{\ttil{} A}.
}
By Chebyshev's inequality, suffices to show $\E \calQ=0$ and $\E|\calQ|^2 = o(n_0^2)$ for all $a$ uniformly in all real symmetric positive semi-definite $A$ of bounded norm. The condition $\E \calQ=0$ is easily verified:
\eq{
    \E \calQ=\sum_{\a,\beta=1}^{m} \E\qa{F_{a\a}F_{a\beta} A_{\a\beta} - \ttil{\a\beta}A_{\beta\a}}=\sum_{\a,\beta=1}^{m} \E\qa{F_{a\a}F_{a\beta} - \ttil{\a\beta}}A_{\beta\a} = 0.
}
For the second condition, we can use the equivalent condition of Cor.~1.1 of \cite{bai2008large} for each choice of $a$.

Note that the definition of the multivariate Gaussian in Eq.~\eqref{eq_mvg} easily extends to more indices, $\a$, $\beta$, $\a'$, $\beta'$, \emph{etc.} In the following calculation it is useful to project out the correlated components of these Gaussian variables, \emph{i.e.} to write
\eq{\label{eq_correlated_gaussian_expansion}
    N_\a = N_\a^{(\a'\beta')} + C_{\a}^{\a'} N_{\a'} + C_{\a}^{\beta'} N_{\beta'},
}
where $N_\a^{(\a'\beta')}$, $N_{\a'}$, and $N_{\a'}$ are mean zero Gaussians such that $N_\a^{(a'\beta')}$ is independent of $N_{\a'}$ and $N_{\beta'}$ and $C_{\a}^{\a'}$ and $ C_{\a}^{\beta'}$ are constants. Moreover, our assumption on $X$ implies that $C_{\a}^{\a'}$ and $ C_{\a}^{\beta'}$ are $\calO(n_0^{\e-1/2})$, or in words, the correlations are small. We also know that $\mathbb{V}[N_\a]-\mathbb{V}[N_\a^{(\a'\beta')}]=\calO(n_0^{\e-1/2})$. This is exactly the same idea as in Sec.~\ref{sec: cor of F} but extended to functions that contain more entries of $F$. We note that Eq.~\eqref{eq_correlated_gaussian_expansion} is not necessarily unique.

For functions $\mathcal{F}$ and $\mathcal{G}$, the general principle is that $\E[\mathcal{F}(F_{a\a_1},\ldots,F_{a\a_k})\mathcal{G}(F_{a\beta_1},\ldots,F_{a\beta_l})]$ is close to zero when $\E\mathcal{F}(F_{a\a_1},\ldots,F_{a\a_k})\cdot \E\mathcal{G}(F_{a\beta_1},\ldots,F_{a\beta_l})=0$ and $\{\a_1,\ldots,\a_k\}\cap \{\beta_1,\ldots,\beta_l\}=\emptyset$. More careful bookkeeping is required when higher-order cancellations are necessary to obtain the leading-order behavior.

The first condition of Cor.~1.1, Eq.~(1.4), is straightforward to verify:
\eq{
    \E\absa{F_{a\a}F_{a\beta} - \ttil{\a\beta}}^2 \leq 4\E F_{a\a}^2F_{a\beta}^2 + 4(\ttil{\a\beta})^2,
}
which is finite by assumption (see Sec.~\ref{sec_pre}) and so is certainly $o(n_0)$ uniformly in $\a$, $\beta$, and $a$.

The next condition, Eq.~(1.5), is more involved. We have to show
\eq{\label{eq_sum_to_bound}
    \sum_{\Lambda} \pa{ \E\pa{ F_{a\a}F_{a\beta} -\ttil{\a\beta} }\pa{ F_{a\a'}F_{a\beta'} - \ttil{\a'\beta'} } }^2 = o(n_0^{2})
}
uniformly in $a$, where
\begin{equation*}
    \Lambda := \{ (\a,\beta,\a',\beta'): 1\leq \a,\beta,\a',\beta'\leq m \} \setminus \{(\a,\beta,\a',\beta'): \a=\a'\neq \beta=\beta' \text{ or } \a=\beta'\neq \a'=\beta \} .
\end{equation*}
We split the sum, Eq.~\eqref{eq_sum_to_bound}, into several pieces:
\al{
    &\sum_\a \qa{ \E \pa{ F_{a\a}^2 - \ttil{\a\a} }^2  }^2, \label{eq_1}\\
    4& \sum_{\a\neq\beta} \qa{\E\pa{ F_{a\a} F_{a\beta} - \ttil{\a\beta} }\pa{ F_{a\a}^2 - \ttil{\a\a} }}^2,   \label{eq_2}\\
    & \sum_{\a\neq\beta} \qa{\E \pa{ F_{a\a}^2 - \ttil{\a\a} } \pa{ F_{a\beta}^2-\ttil{\beta\beta} } }^2,   \label{eq_3}\\
    2&\sum_{\a,\beta,\a'}^{\text{distinct}} \qa{\E \pa{ F_{a\a}^2-\ttil{\a\a} } \pa{ F_{a\beta}F_{a\a'}-\ttil{\beta\a'} } }^2,   \label{eq_4}\\
    4&\sum_{\a,\beta,\a'}^{\text{distinct}} \qa{\E \pa{ F_{a\a}F_{a\beta} - \ttil{\a\beta} } \pa{ F_{a\a}F_{a\a'} -\ttil{\a\a'} } }^2,   \label{eq_5}\\
    \text{and} \quad & \sum_{\a,\beta,\a'\beta'}^{\text{distinct}} \qa{\E \pa{ F_{a\a}F_{a\beta} - \ttil{\a\beta} } \pa{ F_{a\a'}F_{a\beta'} - \ttil{\a'\beta'} } }^2 .  \label{eq_6}
}
These are based on the six possible ways of partitioning the four indices. All indices in the above sums are distinct, and we will see that the addition correlations when some indices are equal are compensated by the lower combinatorial factor from the sum.

\paragraph{Eq.~\eqref{eq_1} is $o(n_0^2)$, since the summands are $o(n_0)$:} We have $\E \pa{ F_{a\a}^2 - \ttil{\a\a} }^2 \leq 4\E  F_{a\a}^4 + 4(\ttil{\a\a})^2$, which is $\calO(1)$ by assumption and thus $o(n_0)$. 

\paragraph{Eq.~\eqref{eq_2} is $o(n_0^2)$, since the summands are $o(1)$:} First consider the term
\begin{equation*}
    \E\pa{ F_{a\a} F_{a\beta} - \ttil{\a\beta} }\pa{ F_{a\a}^2 - \ttil{\a\a} } = \E[(f(N_\a+b_a) f(N_\beta+b_a) - \ttil{\a\beta} )(f(N_\a+b_a)^2 - \ttil{\a\a}) ]
\end{equation*}
and note $N_\beta=N_\beta^{(\a)} + C_\beta^\a N_\a$ then Taylor expand in $C_\beta^\a N_\a$ to get
\als{
    &\E\pa{f(N_\a+b_a) f(N_\beta^{(\a)} +b_a) - \ttil{\a\beta}} \pa{f(N_\a+b_a)^2-\ttil{\a\a} } \\
    &\quad+ C_{\beta}^\a \E N_\a f'(N_{\beta}^{(\a)})  \pa{f(N_\a+b_a)^2-\ttil{\a\a} } +\calO(n_0^{2\e-1}).
}
The second term above is $\calO(n_0^{\e-1/2})$ because of the  $C_\beta^\a$ term. For the first term, note it is equal to 
\eq{
    \E f(N_\beta^{(\a)} +b_a) \cdot \E \pa{f(N_\a+b_a)  } \pa{f(N_\a+b_a)^2-\ttil{\a\a} }.
}
Finally, $\E f(N_\beta^{(\a)} +b_a)=\calO(n_0^{\e-1/2})$ since $C_\beta^\a=\calO(n_0^{\e-1/2})$ and 
\eq{
    0=\E f(N_\beta+b_a) = \E f(N_\beta^{(\a)} + b_a) + C_\beta^\a  \E N_\a f'(N_\beta^{(\a)}) +\calO(n_0^{2\e-1}).
}

So squaring and putting the above bounds together we find the summands are $\calO(n_0^{2\e-1})$, which is $o(1)$.

\paragraph{Eq.~\eqref{eq_3} is $o(n_0^2)$, since the summands are $o(1)$:} Again we write $N_\beta=N_\beta^{(\a)} + C_\beta^\a N_\a$ and Taylor expand in $C_\beta^\a N_\a$:
\als{
    &\E \pa{ f(N_\a +b_a)^2 - \ttil{\a\a} } \cdot \E \pa{ f(N_\beta^{(\a)} +b_a)^2 - \ttil{\beta\beta} } \\
    &\quad+2C_\beta^\a\E \pa{ N_\a f(N_\a +b_a)^2 - \ttil{\a\a} } \cdot \E \pa{ f(N_\beta^{(\a)} +b_a)f'(N_\beta^{(\a)} +b_a) } +\calO(n_0^{2\e-1}),
}
where the first term is zero and the second is $\calO(n_0^{\e-1/2})$ due to $C_\beta^\a$. Therefore after squaring the summands are $\calO(n_0^{2\e-1})$, which is $o(1)$.

\paragraph{Eq.~\eqref{eq_4} is $o(n_0^2)$, since the summands are $o(n_0^{-1})$:} We have 
\eq{\label{eq_12}
    \E \pa{ F_{a\a}^2-\ttil{\a\a} } \pa{ F_{a\beta}F_{a\a'}-\ttil{\beta\a'} } = \E  (f(N_\a+b_a)^2-\ttil{\a\a}) (f(N_\beta+b_a)f(N_{\a'}+b_a)-\ttil{\beta\a'})
}
and expand $N_\beta$ and $N_{\a'}$ in $N_\a$ to get
\als{
    &\E  (f(N_\a+b_a)^2-\ttil{\a\a}) \cdot\E  (f(N_\beta^{(\a')}+b_a)f(N_{\a'}^{(\beta)}+b_a)-\ttil{\beta\a'}) \\
    &\quad+ \E  N_\a (f(N_\a+b_a)^2-\ttil{\a\a})\pa{ C_\beta^\a f'(N_\beta^{(\a)}+b_a)f(N_{\a'}^{(\a)}+b_a)  + C_{\a'}^\a f'(N_{\a'}^{(\a)}+b_a)f(N_{\beta}^{(\a)}+b_a) }\\
    &\quad+ \calO(n_0^{2\e-1}) .
}
The first term above is zero. The second term can be written
\al{\nonumber
    &C_\beta^\a \E  N_\a (f(N_\a+b_a)^2-\ttil{\a\a}) \cdot \E   f'(N_\beta^{(\a)}+b_a)f(N_{\a'}^{(\a)}+b_a) \\  
    & \quad + C_{\a'}^\a \E N_\a (f(N_\a+b_a)^2-\ttil{\a\a}) \cdot \E  f'(N_{\a'}^{(\a)}+b_a)f(N_{\beta}^{(\a)}+b_a).\label{eq_11}
}
Now we can expand $N_\beta^{(\a)}$ in $N_{\a'}$ as $N_\beta^{(\a)} = N_\beta^{(\a\a')} +C_\beta^{\a'}N_{\a'}$, where $N_\beta^{(\a\a')}$ and $N_{\a'}$ are independent. Then applying Lem.~\ref{lem: cor structure}, we note
\eq{
    \E   f'(N_\beta^{(\a)}+b_a)f(N_{\a'}^{(\a)}+b_a) = \E f'(N_\beta^{(\a\a')} +b_a) \cdot \E f(N_{\a'}^{(\a)}+b_a) + \calO(n_0^{\e-1/2}).
}
Applying Lem.~\ref{lem: cor structure} again shows $\E f(N_{\a'}^{(\a)}+b_a)=\calO(n_0^{\e-1/2})$. Thus both terms in Eq.~\eqref{eq_11} are $\calO(n_0^{\e-1/2})$. Combining this with the fact that $C_\beta^\a$ and $C_{\a'}^\a$ are $\calO(n_0^{\e-1/2})$, we conclude that Eq.~\eqref{eq_12} is $\calO(n_0^{2\e-1})$. Squaring shows the summands are $\calO(n_0^{4\e-2})$, which is $o(n_0^{-1})$.

\paragraph{Eq.~\eqref{eq_5} is $o(n_0^2)$, since the summands are $o(n_0^{-1})$:} The argument is similar to that above for Eq.~\eqref{eq_4}. Except we expand as $N_\beta=N_\beta^{(\a)}+C_\beta^\a N_\a$ and  $N_{\a'}=N_{\a'}^{(\a)}+C_{\a'}^\a N_\a$.

\paragraph{Eq.~\eqref{eq_6} is $o(n_0^2)$, since the summands are $o(n_0^{-2})$:}
We expand $N_{\a'}=N_{\a'}^{(\a\beta)} + C_{\a'}^{\a}N_{\a} + C_{\a'}^{\beta}N_{\beta}$ and $N_{\beta'}=N_{\beta'}^{(\a\beta)} + C_{\beta'}^{\a}N_{\a} + C_{\beta'}^{\beta}N_{\beta}$, and then Taylor expand as before. We find
\als{
    &\E \pa{f(N_\a+b_a)f(N_\beta+b_a) - \ttil{\a\beta} } \pa{ f(N_{\a'}+b_a)f(N_{\beta'}+b_a) - \ttil{\a'\beta'} } \\
    &\quad = \E \pa{f(N_\a+b_a)f(N_\beta+b_a) - \ttil{\a\beta}} \cdot \E \pa{ f(N_{\a'}^{(\a\beta)}+b_a)f(N_{\beta'}^{(\a\beta)}+b_a) - \ttil{\a'\beta'} } \\
    &\quad+\E\Bigg{[}\pa{f(N_\a+b_a)f(N_\beta+b_a) - \ttil{\a\beta}}\Big{(}(C_{\a'}^{\a}N_{\a} + C_{\a'}^{\beta}N_{\beta})f(N_{\beta'}^{(\a\beta)} +b_a)f'(N_{\a'}^{(\a\beta)} +b_a) \\
    &\qquad\qquad\qquad+(C_{\a'}^{\a}N_{\a} + C_{\a'}^{\beta}N_{\beta})f(N_{\a'}^{(\a\beta)} +b_a)f'(N_{\beta'}^{(\a\beta)} +b_a) \Big{)}\Bigg{]}.
}
The first term above is zero. For the second term, consider
\eq{\label{eq_22}
    C_{\a'}^{\a}\E N_{\a}\pa{f(N_\a+b_a)f(N_\beta+b_a) - \ttil{\a\beta}} \cdot \E f(N_{\beta'}^{(\a\beta)} +b_a)f'(N_{\a'}^{(\a\beta)} +b_a )
}
which is one of four similar terms that come from expanding the sums out of the second term above. We need to show Eq.~\eqref{eq_22} is $\calO(n^{3\e-3/2})$, and since the other 3 terms are of the same form as Eq.~\eqref{eq_22} this will complete the argument. Taking $\E N_{\a}\pa{f(N_\a+b_a)f(N_\beta+b_a) - \ttil{\a\beta}}$ first and expanding $N_\beta = N_\beta^{(\a)}+C_\beta^{\a} N_\a$, we see
\als{
    \E N_{\a}\pa{f(N_\a\!+\!b_a)f(N_\beta\!+\!b_a) \!-\! \ttil{\a\beta}} &= \E N_{\a}f(N_\a+b_a)f(N_\beta+b_a) \\
    &= \E N_{\a}f(N_\a\!+\!b_a)\E f(N_\beta^{(\a)}\!+\!b_a) \!+\! C_\beta^{\a}\E N_{\a}^2f(N_\a\!+\!b_a)f'(N_\beta^{(\a)}) \!+\! \calO(n_0^{2\e\!-\!1}) \\
    &=\calO(n_0^{\e-1/2}),
}
since $\E f(N_\beta^{(\a)}+b_a)=\calO(n_0^{\e-1/2})$ as above. Second, we consider $\E f(N_{\beta'}^{(\a\beta)} +b_a)f'(N_{\a'}^{(\a\beta)} +b_a )$ by expanding $N_{\beta'}^{(\a\beta)} = N_{\beta'}^{(\a\beta\a')} + C_{\beta'}^{\a'}N_{\a'}$
\als{
    \E f(N_{\beta'}^{(\a\beta)} \!+\!b_a)f'(N_{\a'}^{(\a\beta)} \!+\!b_a ) & \!= \E f(N_{\beta'}^{(\a\beta\a')} \!+\!b_a) \E f'(N_{\a'}^{(\a\beta)} \!+\!b_a ) \!+\! C_{\beta'}^{\a'}\E f'(N_{\a'}^{(\a\beta)} \!+\!b_a ) N_{\a'}f' \!+\! \calO(n_0^{2\e\!-\!1}),
}
which is $\calO(n_0^{\e-1/2})$ since both $\E f(N_{\beta'}^{(\a\beta\a')} +b_a)$ and $ C_{\beta'}^{\a'}$ are $\calO(n_0^{\e-1/2})$ as before. Third, ${C_{\a'}^{\a} = \calO(n_0^{\e-1/2})}$, so combining these three multiplicative factors we have Eq.~\eqref{eq_22} is $\calO(n_0^{3\e-3/2})$.

Finally, squaring we see the summands are $\calO(n_0^{6\e-3})$, which is clearly $o(n_0^{-2})$.

\paragraph{Assumption (A4).} This is much easier to verify: using Lem.~\ref{lem: cor structure}, we see
\als{
    \tr{(\ttil{})^2} &= \sum_{\a,\beta} (\ttil{\a\beta})^2 \\
    &= \sum_{\a} (\ttil{\a\a})^2 + \sum_{\a\neq\beta} (\ttil{\a\beta})^2 \\
    &= n_0 \eta(b_a)^2 + \zeta(b_a)^2 \sum_{\a\neq \beta}c_{\a\beta}^2 + \calO(n_0^{3\e+1/2}) \\
    & = \calO(n_0^{2\e+1})\\
    &=o(n_0^2)
}
without any dependence on $a$.

\section{Derivation of self-consistent equation for \texorpdfstring{$\Flin$}{}}\label{sec: sce derivation}

\begin{theorem}
    Let $m_{m}(z)$ be the Stieltjes transform of $\frac{1}{n_1}\Flin^\top \Flin$, \emph{i.e.} $\frac{1}{m}\tr G(z)$. Then with probability 1, as $n_1\to\infty$, for all $z$ such that $\Im z>0$, $m_{n_1}(z)\to s(z)$, where $s(z)$ is the solution of the coupled equations
    \eq{
    s(z) = \E_{S \sim MP(\phi)}\qa{
        \frac{1}{C_0(z) + S C_1(z)}
    } \quad\text{and}\quad 
    \tilde{s}(z) = \E_{S \sim MP(\phi)}\qa{
        \frac{S}{C_0(z) + S C_1(z)}
    }
}
with
\begin{equation}
    C_0(z) := -z + \E_{B\sim \mu_B} \qa{ \frac{\eta(B) - \zeta(B)}{D(B)} },\quad
   C_1(z) := \E_{B\sim \mu_B} \qa{ \frac{\zeta(B)}{D(B)}}\,,
\end{equation}
\begin{equation}
 D(B) := 1+\frac{\psi}{\phi} \pa{ \zeta(B)\tilde{s}(z) +\pa{\eta(B)-\zeta(B)}s(z) } \,.
\end{equation}
\end{theorem}

We want to study the eigenvalues of $\frac{1}{n_1}\Flin^\top \Flin$. Note without loss of generality it is sufficient to consider diagonal $\Sigma$, since we can diagonalize $\Sigma$ using some orthogonal matrices $O$ and $O'$ as it is the square root of a positive definite matrix. Moreover, these orthogonal matrices, when applied to either $\Theta^1$ or $\Theta^2$ do not change their distributions. Thus, diagonalizing $\Sigma$, so that $\Sigma_{\a\a}=\sqrt{\lambda_{\a}^X}$, where $\lambda_\a^X$ are the eigenvalues of $X^\top X/n_0$, results in an equivalent matrix ensemble in distribution (see \cite{silverstein1995empirical} for more detail). With this simplification, $\Flin$ has independent entries given by
\eq{
    F^{\text{lin}}_{a\a} = \mathcal{C}_{a}\Theta^1_{a\a} \sqrt{\lambda_\a^X}+\sqrt{\mathcal{V}_{a}^2-\mathcal{C}_{a}^2}\Theta^2_{a\a}.
}

To make the derivation easier, we can partly linearize the problem by studying the matrix
\eq{
    H := \begin{bmatrix}
    -z I        &    \Flin^\top / \sqrt{n_1} \\
    \Flin / \sqrt{n_1}  &   -I 
\end{bmatrix}.
}

By the Schur complement formula, one easily finds
\begin{align}\label{eq: sce 1}
    s_m(z) &:= \frac{1}{m} Tr(\Flin^\top \Flin/n_1  - zI) = \frac{1}{m}\sum_{\a=1}^{m} G_{\a\a}(z) \\
    \label{eq: sce 2}\text{and} \quad z \tilde{s}_m(z) &:= \frac{1}{n_1} Tr(\Flin\Flin^\top /n_1 - zI) = \frac{1}{n_1}\sum_{a=n_1+1}^{n_1+m} G_{a a}(z),
\end{align}
where $G$ is the \emph{inverse} of $H$. Again by the Schur complement formula
\begin{align}\label{eq: sce 3}
    \frac{1}{G_{\a\a}} & = -z -  \sum_{a,b=1}^{n_1} F^{\text{lin}}_{a \a}F^{\text{lin}}_{b \a} G^{(\a)}_{m+a, m+b} \\
    \label{eq: sce 4}\text{and}\quad \frac{1}{G_{aa}} & =-1 -  \sum_{\a,\beta=1}^{m} F^{\text{lin}}_{a\a}F^{\text{lin}}_{a\beta} G^{(a)}_{\alpha\beta}
\end{align}
for $\alpha\in\{1,\ldots, m\}$ and $a\in\{m+1,\ldots, m+n_1 \}$ and $G^{(a)}$ is the inverse of the minor $H^{(a)}$. Since $G^{(\a)}$ is independent of $\Theta^1_{1\a},\ldots,\Theta^1_{n1\a}$ and $\Theta^2_{1\a},\ldots,\Theta^2_{n1\a}$, we see by taking the expectation over these variables that
\eq{
    \E \sum_{a,b=1}^{n_1} F^{\text{lin}}_{a \a}F^{\text{lin}}_{b \a} G^{(\a)}_{m+a, m+b} = \frac{1}{n_1} \sum_{a=1}^{n_1} \pa{\lambda_{\a}^X \zeta(b_a) + \eta(b_a)- \zeta(b_a) } G^{(\a)}_{m+a,m+a}. 
}
Moreover, standard concentration inequalities and the Ward identity (see \cite{erdos2017dynamical}) show 
\al{
    &\absa{ \sum_{a,b=1}^{n_1} F^{\text{lin}}_{a \a}F^{\text{lin}}_{b \a} G^{(\a)}_{m+a, m+b} - \E \sum_{a,b=1}^{n_1} F^{\text{lin}}_{a \a}F^{\text{lin}}_{b \a} G^{(\a)}_{m+a, m+b} }  \\
    & \leq C \pa{ \frac{1}{m^2}\sum_{a,b}\absa{G^{(\a)}_{m+a, m+b}}^2 }^{1/2} \leq C \pa{ \frac{1}{m^2}\sum_{a}\frac{\Im G^{(\a)}_{m+a, m+a}}{\Im z} }^{1/2} \leq \calO(1/\sqrt{m})
}
with high probability. Similar bounds are easily obtained for $\sum_{\a,\beta=1}^{m} F^{\text{lin}}_{a\a}F^{\text{lin}}_{a\beta} G^{(a)}_{\alpha\beta}$.

We may also replace $G^{(\a)}$ with $G$ at the expense of another small error that can be bounded using the Cauchy interlacing theorem: $\absa{G^{(\a)}_{m+a,m+a}-G_{m+a,m+a}}\leq \calO(1/m)$. Using this control over these sums, we see
\al{\label{eq: sce 5}
	\frac{1}{G_{\a\a}} &=-z- \sum_{a,b=1}^{n_1}F^{\text{lin}}_{a \a}F^{\text{lin}}_{b \a} G^{(\a)}_{ab} \\
	&= -z -\frac{1}{n_1} \sum_{a=1}^{n_1} \pa{\lambda_{\a}^X \zeta(b_a) + \eta(b_a)- \zeta(b_a) } G_{m+a,m+a} + \calO(1/\sqrt{m})
 }
and
\al{\label{eq: sce 6}
	\frac{1}{G_{aa}} &=-1- \sum_{\a,\beta=1}^{m}F^{\text{lin}}_{a \a}F^{\text{lin}}_{a\beta} G^{(a)}_{\a\beta} \\
	&= -1 -\frac{1}{n_1} \sum_{\a=1}^{m} \pa{\lambda_{\a}^X \zeta(b_a) + \eta(b_a)-\zeta(b_a) } G_{\a\a} + \calO(1/\sqrt{m}). 
}

Finally, we invert Eq.~\eqref{eq: sce 6}, multiply by ${\lambda_{\a}^X \zeta(b_a) + \eta(b_a) -\zeta(b_a)}$, and average over $a$ to find
\al{
	&\frac{1}{n_1} \sum_{a=1}^{n_1} \pa{\lambda_{\a}^X \zeta(b_a) + \eta(b_a) -\zeta(b_a)} G^{(\a)}_{m+a,m+a} \\
    &=  - \frac{1}{n_1} \sum_{a=1}^{n_1} \frac{\lambda_{\a}^X \zeta(b_a) + \eta(b_a)-\zeta(b_a) }{1 +\frac{\psi}{\phi}\pa{ \zeta(b_a) \tilde{s}(z) +  \pa{ \eta(b_a) -\zeta(b_a)}s_m(z) }+\calO(1/\sqrt{m})} \\
	&=  - \frac{1}{n_1} \sum_{a=1}^{n_1} \frac{\lambda_{\a}^X \zeta (b_a) + \eta(b_a) -\zeta(b_a)}{1 +\frac{\psi}{\phi}\pa{ \zeta(b_a) \tilde{s}(z) +  (\eta(b_a)-\zeta(b_a) s_m(z) }} +\calO(1/\sqrt{m})\\
	&= -\E_{B\sim \mathcal{N}(0,\sigma_b^2) }\qa{ \frac{\lambda_{\a}^X \zeta (B) + \eta(B) -\zeta(B)}{1 +\frac{\psi}{\phi}\pa{ \zeta(B) \tilde{s}(z) +  (\eta(B)-\zeta(B)) s_m(z) }}}+o(1),
}
where $\tilde{s}_m(z)=\frac{1}{m}\sum_\a \lambda_{\a}^X G_{\a\a}$, where we Taylor expanded in the second step, and where we used our assumption on $B$.

We can now invert Eq.~\eqref{eq: sce 5} and average over $\a$ to find
\eq{\label{eq: sce m}
	s_m(z) \to \E_{S\sim\mu_X}\qa{ \frac{1}{-z +\E_{B\sim \mathcal{N}(0,\sigma_b^2) }\qa{ \frac{S \zeta (B) + \eta(B)-\zeta(B) }{1 +\frac{\psi}{\phi}\pa{ \zeta(B) m_\sigma(z) +  (\eta(B) -\zeta(B))s(z) }}}}}.
}
Similarly,
\eq{\label{eq: sce mtilde}
	\tilde{s}_m(z)  \to \E_{S\sim\mu_X}\qa{ \frac{S}{-z +\E_{B\sim \mathcal{N}(0,\sigma_b^2) }\qa{ \frac{S \zeta (B) + \eta(B)-\zeta(B) }{1 +\frac{\psi}{\phi}\pa{ \zeta(B) \tilde{s}(z) +  (\eta(B)-\zeta(B) s(z) }}}}}.
}

Note that the integral over $S$ here is an integral over the limiting distribution of the data $\mu_X$, which we assume
\eq{
    \frac{1}{m}\sum_{\a} \delta_{\lambda_\a^X} \to \mu_X
}
in distribution. In the case of iid Gaussian data, this is exactly given by the Marchenko-Pastur distribution.

\section{Nonlinear mixtures can increase model capacity over single nonlinearities}
\label{sec_training_mixtures_better}
In this section, we build on our results for the training loss on noisy autoencoder tasks to examine the benefits of utilizing nonlinear mixtures. For this analysis, we consider the simplest possible nontrivial distribution over activation functions: a Bernoulli mixture of two different functions. To each of these functions we associate two constants, $\eta$ and $\zeta$, which derive from Eq.~\eqref{eqn:eta_zeta} but have no $B$-dependence since each function is a \emph{single nonlinearity}. Concretely, let
\begin{equation}
    \eta = \E\qa{f(N)^2} \quad \text{and} \quad
    \zeta = \E\qa{\sigma_Z f'(N)}^2,
\end{equation}
where $N \sim \mathcal{N}(0, \sigma_Z)$. For the two functions themselves, we utilize (i) a "pure linear" activation function with $\eta=1$, \emph{i.e.} the identity function and (ii) a "pure nonlinear"~\citep{hastie2019surprises} activation function with $\eta = 1$ and $\zeta = 0$. (The particular purely nonlinear function in (ii) is irrelevant, as our theory predicts and our experiments confirm; see Fig.~\ref{fig:mixture_sm}(a)). To be precise, we define for $p\sim \text{Bernoulli(p)}$, 
\eq{
    f_p(x) := \begin{cases}
        x & \text{if } p = 0\\
        g_{\zeta = 0}(x) & \text{if } p = 1 \end{cases} ,
}
where $g_{\zeta = 0}$ is any function with $\eta = 1$ and $\zeta = 0$ (see, \emph{e.g.}, the functions in Fig.~(3) of~\citep{pennington2017nonlinear}).

The task of computing $E_{\text{train}}$ for $f_p$ is cumbersome but purely algebraic. To see how to proceed, notice that the expectations in Eq.~\eqref{eq: main sce aux} are simple for $f_p$:
\eq{
    C_0(z) = -z + \frac{p}{1+\psi/\phi\,s(z)} \quad\text{and}\quad
    C_1(z) = \frac{1-p}{1+\psi/\phi\,\tilde{s}(z)}\,.
}
Plugging these equations into Eqs.~\eqref{eq:m_reduced} and \eqref{eq:mtilde_reduced}, collecting terms and simplifying yields a set of coupled polynomial equations for $s(z)$ and $\tilde{s}(z)$. Taking the total derivative of these equations with respect to z yields two additional equations which can be solved to express $s'(z)$ and $\tilde{s}'(z)$ in terms of $s(z)$ and $\tilde{s}(z)$. Combining these results produces a polynomial system whose solution\footnote{Special care must be taken in selecting the correct root of this equation, in accordance with the condition that $s(z)\sim-1/z$ for large $|z|$.} encodes $E_\text{train}$ through Eq.~\eqref{eq: training error}.
Fig.~\ref{fig:mixture_sm}(a) shows the result of this calculation in solid lines for various values of $\gamma$, while the $1\sigma$ error bars show empirical simulations with finite networks. The red stars in the figure show that for many values of $\gamma$, the optimal mixture percentage is intermediate, \emph{i.e.} $0<p<1$.

\begin{figure*}[t]
    \centering
    \includegraphics[width=\linewidth]{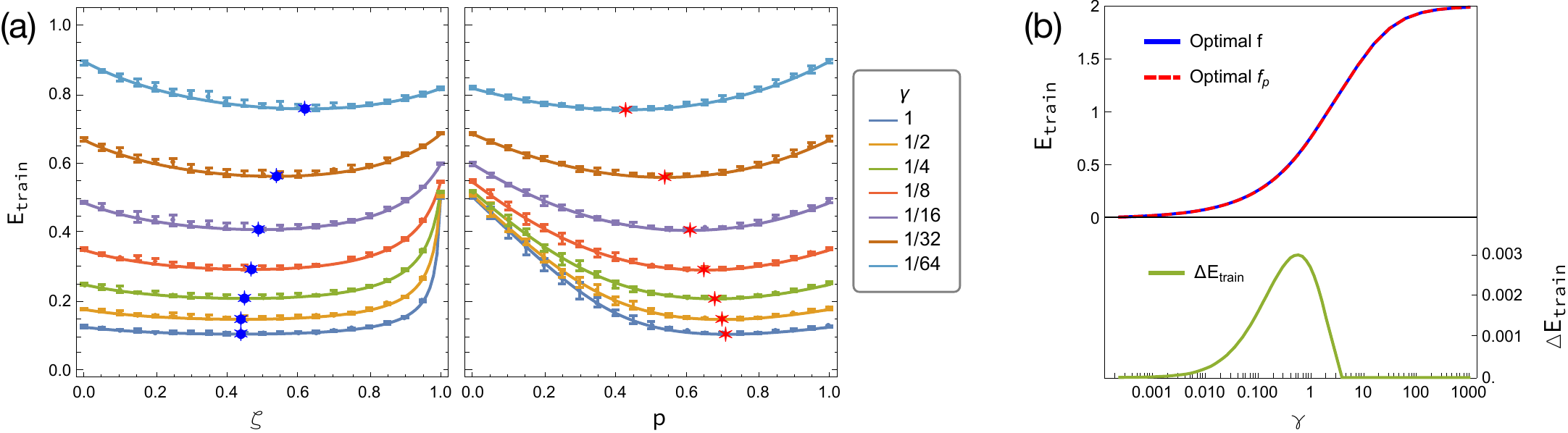}
    \caption{Performance on ridge-regularized noisy autoencoder with $\sigma_\epsilon = 1$, $\phi=1/2$, and $\psi = 1/2$. (a) Theoretical predictions for training error (solid lines) and $1\sigma$ error bars for empirical simulations of finite networks ($n_0=192$, $n_1=384$, $m=384$) for various values of ridge regularization constant $\gamma$ as the activation function varies. In the left panel, a single activation function $f$ is used. In the right panel, the non-linearity is $f_p$, a $\text{Bernoulli}(p)$-mixture of a purely linear ($\zeta=1$) and purely non-linear ($\zeta=0$) function. Each simulation uses a randomly-chosen non-linearity having the specified values of $\zeta$, demonstrating that $E_{\text{train}}$ depends on the non-linearity solely through this constant. Red and blue stars denote minima. (b) Training error as a function of $\gamma$ for the optimal $f$ and $f_p$, as determined in (a). The bottom panel shows the difference in training error, demonstrating that the optimal $\text{Bernoulli}(p)$-mixture of non-linearities has smaller training error than the best single non-linearity.}
    \label{fig:mixture_sm}
\end{figure*}

The question is, does a non-trivial mixture actually outperform a single nonlinearity? First, we must understand the performance of the optimal single nonlinearity. We note that owing to the homogeneity of the the training loss in $\eta$, $\zeta$, and $\gamma$, we can assume without loss of generality that $\eta = 1$. Therefore the entire effect of the nonlinearity should be encoded in the single constant $\zeta$. In Fig.~\ref{fig:mixture_sm}(a), we plot our theoretical prediction for $E_\text{train}$ in solid lines and empirical simulations for finite networks as $1\sigma$ error bars. The activation function used for each simulation is chosen randomly, conditional on the value of $\zeta$. So the good agreement in the left panel of Fig.~\ref{fig:mixture_sm}(a) demonstrates not just the correctness of our theoretical result but also the fact that $E_\text{train}$ depends on the activation function solely through the constant $\zeta$. The blue stars in this figure indicate that the optimal single nonlinearity is neither purely linear ($\zeta = 1$), nor purely nonlinear ($\zeta = 0$), but rather something in between.

For this particular problem setup, the performance of the optimal single nonlinearity and the optimal Bernoulli mixture are rather close, as indicated by the top panel of Fig.~\ref{fig:mixture_sm}(b). However, owing to our precise analytical formulation, we can evaluate the training loss to high precision and observe that there is indeed a difference in performance between the two models, as shown in the bottom panel of Fig.~\ref{fig:mixture_sm}(b). This result establishes that there are some problems for which even the best single nonlinearity is outperformed by a mixture of nonlinearities.

\section{Derivation of the test error}
\label{sec_test_error_sm}

The linearization from Sec.~\ref{subsec: Flin} is a key ingredient in the calculation of the test error. In the bias free case, \cite{adlam2020neural} shows how it can be used with operator-valued free probability to provide an asymptotically exact prediction for the test error under the data generating assumption of Sec.~\ref{sec_test_error}. We briefly review this method here, highlighting the differences.

We use a slightly different linearization here than Sec.~\ref{subsec: Flin} (which was more convenient for the derivation in Sec.~\ref{sec: sce derivation}), but their correlation structure is easily verified to be identical. We set
\eq{
	\Flin := \mathcal{C} W X + \pa{\mathcal{V}^2-\mathcal{C}^2}^{1/2} \Theta^2.
}

The test error is defined as 
\eq{
    E_\text{test} = \E_{\beta,\epsilon} \E_\mathbf{x} \pa{ y(\bfx) - W_2^* f(W \mathbf{x} ; B) }^2.
}
As in Sec.~4 of \cite{adlam2020neural}, the test error can be expressed as the sum of terms $E_1$, $E_2$, and $E_3$, where
\al{
    E_1 &= \mathbb{E}_{\beta,\bfx,\varepsilon}\tr(y(\bfx)y(\bfx)^\top), \\
    E_2 &= -2\mathbb{E}_{\beta,\bfx,\varepsilon}\tr(K_\bfx^\top K^{-1}Y^\top y(\bfx)), \\
    E_3 &= \mathbb{E}_{\beta,\bfx,\varepsilon}\tr(K_\bfx^\top K^{-1}Y^\top Y K^{-1} K_\bfx),
}
$K=F^\top F +\gamma I$, $K_\bfx = F^\top f(W \mathbf{x} ; B)$, $Y=\beta^\top X+\e$, and $y(\bfx) = \beta^\top\bfx$. Using these equalities the expectation over $\bfx$, $\beta$, and $\epsilon$ can be calculated to find \begin{align}
    E_1 & = 1\\
    E_2 &= E_{21}\label{eqn:E2}\\
    E_3 &= E_{31} + E_{32} \label{eqn:E3}\,,
\end{align}
where,
\begin{align}
    E_{21} & = -2 \frac{1}{n_0^{3/2}n_1} \E \tr \left(X^\top   W^\top \mathcal{C} FK^{-1}\right)\label{eq_E2}\\
    E_{31} &= \sigma_{\varepsilon}^2 \E\tr\left(K^{-1} \Sigma_3 K^{-1}\right)\label{eq_E31}\\
    E_{32} &= \frac{1}{n_0} \E \tr\left(K^{-1} \Sigma_3 K^{-1} X^\top X\right)\label{eq_E32}
\end{align}
and,
\eq{
    \Sigma_3 = \frac{1}{n_0 n_1^2}F^\top \mathcal{C}W  W^\top\mathcal{C} F + \frac{1}{n_1^2} F^\top (\mathcal{V}^2-\mathcal{C}^2)F\,.
}

We leave out the details because the calculations proceed in exactly the same fashion as in Sec.~4 of \cite{adlam2020neural}, and the end result is in fact the same as Eq.~(27) of \cite{adlam2020neural}: asymptotically $E_\text{test}$ is equal to the generalized cross-validation (GCV) metric of~\cite{golub1979generalized}. We remark that this correspondence to GCV strongly suggests a deeper underlying connection for random feature methods, and indeed similar results on the asymptotic correctness of the GCV error have been found in \cite{hastie2019surprises,jacot2020kernel}. We leave the investigation of this general connection
to future work.

\section{Additional figures}

\begin{figure}[ht]
    \centering
    \begin{tabular}{cc}
    \includegraphics[width=0.35\linewidth]{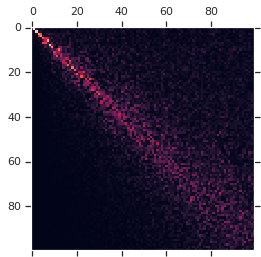} & 
    \includegraphics[width=0.35\linewidth]{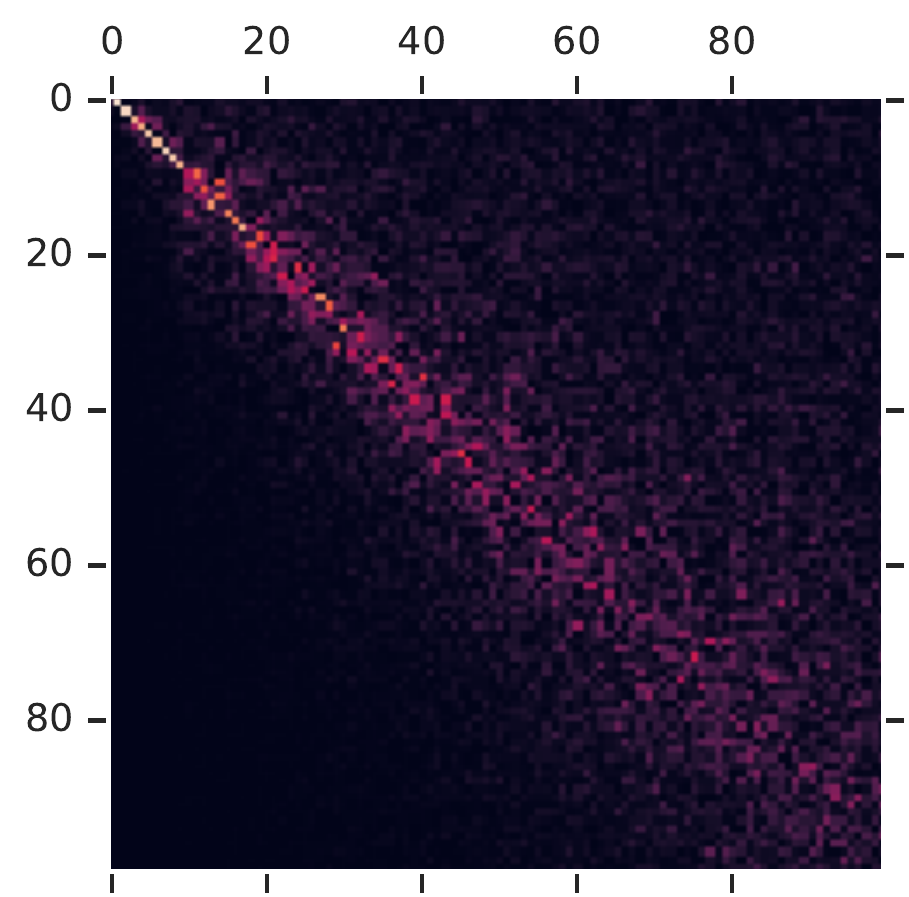}
    \end{tabular}
    
    \caption{Comparison of right singular vectors of $X$ and $F$ for MNIST (left) and CIFAR10 (right). Entry $ij$ shows $\mathbf{x}_i \cdot \mathbf{f}_j$. Although our theoretical results do not give predictions for how the singular vectors change, we found interesting behavior, with very little change to the largest singular vectors (which are nearly isolated in the spectrum), but more mixing of singular vectors in the dense part of the distribution.}
    \label{fig:mnist_singular_vectors}
\end{figure}

\begin{figure}[ht]
    \centering
    \includegraphics[width=\linewidth]{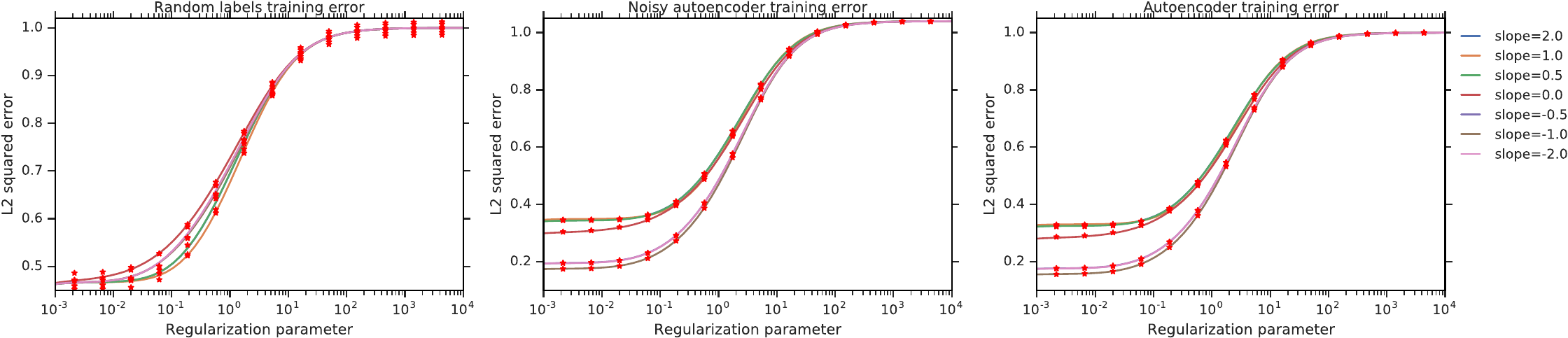}
    \caption{Comparisons of simulated ridge regression error and our theoretical prediction. Here we vary the activation function by changing the slope $\alpha$ in leaky ReLU. In particular $\alpha=-1$ is a linear function, $\alpha=0$ is regular ReLU, and $\alpha=1$ is the a scaled absolute value function. We normalize all functions so that $\mathbb{E}_b[\eta(b)] = 1$. We get excellent agreement with theory from only a single sample.}
    \label{fig:slope_ridge_reg}
\end{figure}

\begin{figure}[ht]
    \centering
    \includegraphics[width=0.5\linewidth]{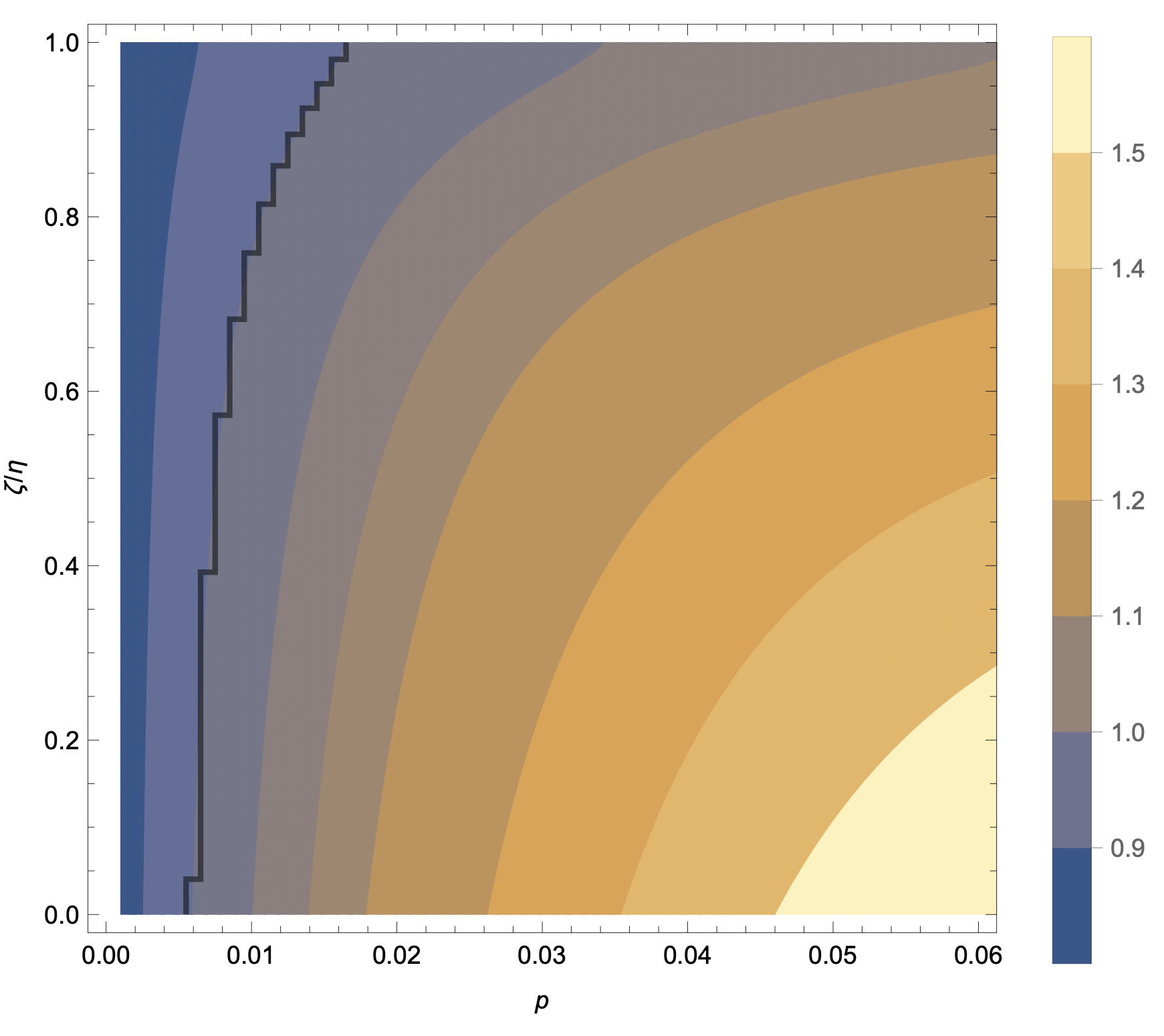}
    \caption{A close up of the left-hand side of Fig.~\ref{fig_mixtures} \textbf{Right} where $p$ is small.}
    \label{fig_mixture_zoom}
\end{figure}

\end{document}